\documentclass[10pt,twocolumn,letterpaper]{article}
\usepackage[pagenumbers]{cvpr}
\usepackage[utf8]{inputenc}
\usepackage[T1]{fontenc}
\usepackage{microtype}
\usepackage{amsthm}
\usepackage{multirow}
\usepackage{makecell}
\usepackage{array}
\usepackage{tabularx}
\usepackage{colortbl}
\usepackage{bm}
\usepackage{cuted}
\usepackage{placeins}

\definecolor{cvprblue}{rgb}{0.21,0.49,0.74}
\usepackage[breaklinks,colorlinks,allcolors=cvprblue,hyperfootnotes=false]{hyperref}
\newtheorem{proposition}{Proposition}
\title{OmniCam: Omni-Camera Trajectory Generation\\ via Geometry-Grounded Pose Token Learning}
\author{Zhenyang Liu$^{1,2,3}$\quad Chenjie Cao$^{3}$\quad Yisu Zhang$^{4,3}$\quad Xuhui Zuo$^{3}$\\
Xiangyang Xue$^{1,\ddagger}$\quad Yanwei Fu$^{1,2,\ddagger}$\quad Tengfei Wang$^{3,\dagger}$\quad Chunchao Guo$^{3}$\\[3pt]
{\normalsize $^{1}$Fudan University\quad $^{2}$Shanghai Innovation Institute\quad $^{3}$Tencent Hunyuan\quad $^{4}$Zhejiang University}\\[3pt]
{\large Project Page: \href{https://zhenyangliu.github.io/OmniCam-Website/}{ZhenyangLiu.github.io/OmniCam-Website}}}

\makeatletter
\patchcmd{\@maketitle}{\vskip .5em}{\vskip 0pt}{}{\PackageError{omnicam}{Title spacing patch failed}{}}
\patchcmd{\@maketitle}{\vspace*{12pt}}{\vspace*{0pt}}{}{\PackageError{omnicam}{Title spacing patch failed}{}}
\makeatother
\hypersetup{pdftitle={OmniCam: Omni-Camera Trajectory Generation via Geometry-Grounded Pose Token Learning},pdfsubject={Computer Vision and Pattern Recognition},pdfauthor={Zhenyang Liu, Chenjie Cao, Yisu Zhang, Xuhui Zuo, Xiangyang Xue, Yanwei Fu, Tengfei Wang, Chunchao Guo}}

\begin{document}
\maketitle

\begin{strip}
    \centering
    \captionsetup{type=figure,hypcap=false}
    \includegraphics[width=1.0\linewidth]{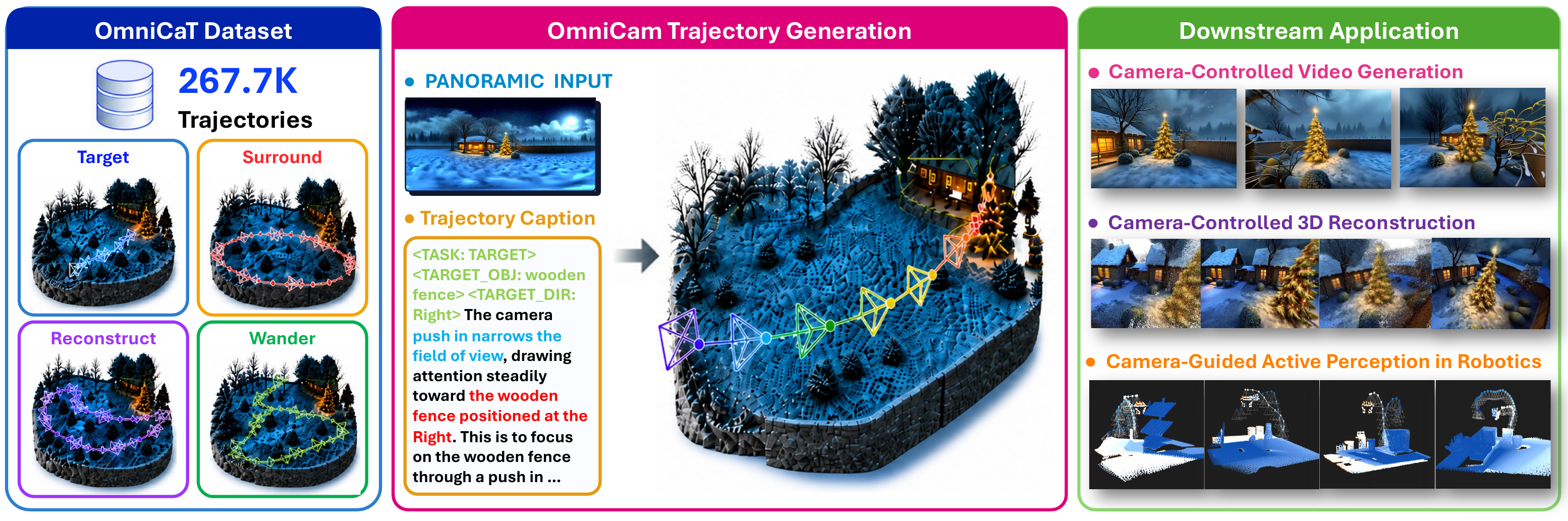}
    \captionof{figure}{
        \textbf{Overview of OmniCam.}
        Given a single panorama and textual trajectory descriptions, OmniCam generates spatially aware camera trajectories (SE(3) pose sequences).
        OmniCam supports four camera behaviors that enable various downstream applications such as camera-controlled video generation, 3D reconstruction, and active perception in robotics.
    }
    \label{fig:teaser}
\end{strip}

\begingroup
\renewcommand{\thefootnote}{}
\footnotetext{$^{\dagger}$Project Lead.\quad $^{\ddagger}$Corresponding authors.}
\endgroup

\begin{abstract}
Camera trajectories control viewpoint changes in video generation, scene reconstruction, and robotic perception. Generating them from language requires both scene geometry and target-aware framing.
We introduce \textbf{OmniCam}, an autoregressive model that generates camera pose sequences from a single panorama and textual trajectory descriptions. Its geometry-grounded pose token learning combines three components: a panoramic point-cloud encoder for omnidirectional geometric context; hybrid absolute-rotation and relative-translation tokenization with temporally consistent quaternion signs; and separate geometric and semantic conditioning streams with an explicit 3D target anchor.
We also construct \textbf{OmniCaT}, containing 267,700 trajectories across four camera behaviors. On the reported OmniCaT evaluation, OmniCam reduces trajectory errors by 28--47\% and collision rate by 65.8\% relative to GenDoP retrained on OmniCaT. Against the best baseline for each metric, the ATE and collision reductions are 43.0\% and 62.3\%, respectively. Component ablations support the use of geometric and target-aware conditioning, while downstream experiments examine camera-controlled video generation and robotic active perception.

\end{abstract}

\section{Introduction}
\label{sec:intro}
Camera trajectories serve as a control interface for spatial-intelligence systems, including world models~\cite{alonso2024diffusion}, video generation~\cite{he2024cameractrl,wan2025wan}, and robotics~\cite{intelligence2025pi,kim2024openvla,bar2025navigation}. Specifying a trajectory requires deciding \emph{where to move}, \emph{how to orient}, and \emph{what to attend to}. A prescribed orbit or linear sweep may fail when it intersects scene geometry or loses the intended target. We study generating scene-aware camera trajectories from a panoramic observation and a textual instruction (\Cref{fig:teaser}).

Recent trajectory generators~\cite{jiang2024cinematographic,courant2024exceptional,li2024director3d,zhang2025gendop} motivate three design questions.
\emph{(i)~Scene coverage}: perspective RGBD inputs~\cite{zhang2025gendop} provide limited angular coverage, while geometry-based planners require explicit spatial representations~\cite{putz20163d,ruetz2019ovpc}. A panorama provides observations in all directions, although occluded surfaces remain unknown.
\emph{(ii)~Pose serialization}: quantization trades precision against token vocabulary size. Quaternion sign ambiguity ($q\equiv -q$), coordinate scale, and the accumulation of relative increments can affect the resulting learning problem.
\emph{(iii)~Conditioning}: geometric free-space information and language-referenced targets serve different roles in planning. Providing dedicated pathways for these cues is a useful architectural choice whose value should be tested empirically.
We use \emph{geometry-grounded pose token learning} to describe a design that connects geometric perception, pose serialization, and target-aware conditioning.

\textbf{OmniCam} implements this design with a panoramic point-cloud encoder, hybrid pose tokenization, and a dual-branch conditioning architecture. The point cloud, reconstructed from estimated panoramic depth, provides omnidirectional but partial geometry. Rotations are encoded absolutely with sequential quaternion sign alignment; translations are encoded as increments in a scene-normalized coordinate system. Sign alignment selects consistent representatives of the quaternion double-cover, while quantization remains lossy. Absolute rotations avoid composition of rotation-quantization residuals; relative translations trade a smaller per-step range against accumulated displacement error. We analyze these distinctions in \Cref{sec:theory_appendix} and compare the encodings empirically.

Separate geometric and semantic pathways supply the decoder with structural and target-related features. An attention-derived 3D target centroid is injected through gated cross-attention to provide a spatial anchor for object-directed trajectories. The geometry pathway does not directly receive text; this routing does not imply statistical independence or guarantee the absence of interference after fusion. At inference, a format constraint restricts termination to complete pose-token groups.

Training requires trajectories paired with scene observations and instructions. Existing collections~\cite{yu2023mvimgnet,zhou2018stereo,ling2024dl3dv,zhang2025gendop} cover recorded object, scene, and cinematic motions. We construct \textbf{OmniCaT} by recovering geometry and semantic targets from panoramas and synthesizing four behaviors: Target, Surrounding, Reconstruction, and Wandering. The dataset contains 267,700 trajectories with hierarchical descriptions of task, target, and motion. Because supervision is generated on estimated geometry, the learned model may inherit reconstruction errors and planner preferences.

Our contributions are:
\begin{itemize}
\item \textbf{Panorama-conditioned trajectory generation.} OmniCam combines explicit point-cloud features, dedicated geometric and semantic pathways, and a 3D target anchor to generate camera trajectories from panoramic imagery and language.
\item \textbf{Hybrid pose serialization.} We combine absolute rotations, relative translations, and consistent quaternion signs; provide qualified error and representation analyses; and compare the choices through ablations.
\item \textbf{OmniCaT and empirical evaluation.} We introduce a 267,700-trajectory dataset spanning four behaviors. The reported OmniCaT results improve on retrained baselines, with component ablations and downstream video and robotic-perception evaluations.
\end{itemize}

\section{Related Work}
\noindent\textbf{Camera Trajectory Generation.}
Early approaches rely on optimization-based planning~\cite{blinn2002looking, lino2015intuitive, galvane2015camera, liu2024splatraj} or learning-based control~\cite{bonatti2020autonomous, drucker1992cinema, huang2019learning, jiang2020example}.
Recent methods integrate camera movement with scene dynamics: CCD~\cite{jiang2024cinematographic} employs a diffusion model with text and keyframe conditioning in character-centric coordinates; E.T.~\cite{courant2024exceptional} incorporates character trajectories in global coordinates; Director3D~\cite{li2024director3d} adopts a DiT-based framework on multi-view data; Hou et al.~\cite{hou2024learning} predict next-frame motion autoregressively for aerial videography; GenDoP~\cite{zhang2025gendop} combines text with RGBD inputs in an autoregressive framework.
These methods use different task-specific inputs and trajectory representations. OmniCam studies panoramic point-cloud conditioning and explicit target anchoring, together with the effect of absolute and relative pose serialization in an autoregressive generator.

\begin{figure*}[tp]
    \centering
    \includegraphics[width=1.0\linewidth]{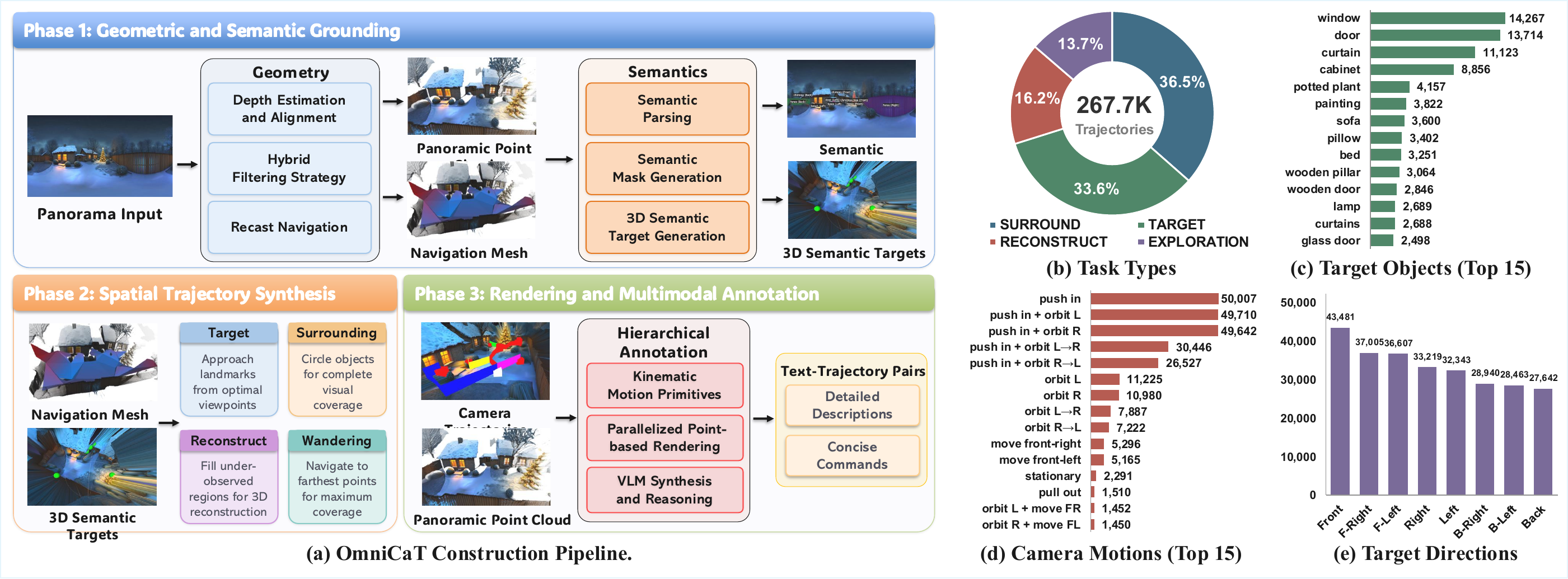}
    \caption{\textbf{OmniCaT Dataset Overview.} \emph{Left}: (a) Construction pipeline comprising (1) Geometric and Semantic Grounding to extract nav-meshes and 3D targets, (2) Spatial Trajectory Synthesis to generate diverse camera behaviors, and (3) Rendering and Multimodal Annotation to produce text-trajectory pairs. \emph{Right}: Dataset statistics showing (b) task distribution, (c) top object categories, (d) motion primitive diversity, and (e) target direction distribution.}
    \label{fig:omnicat}
\end{figure*}

\noindent\textbf{Camera Trajectory Datasets.}
Existing trajectory datasets remain limited in spatial diversity. MVImgNet~\cite{yu2023mvimgnet}, RealEstate10K~\cite{zhou2018stereo}, and DL3DV-10K~\cite{ling2024dl3dv} provide calibrated trajectories but focus on simple motions around stationary objects or within static environments. CCD~\cite{jiang2024cinematographic} and E.T.~\cite{courant2024exceptional} offer character-centric tracking data that does not generalize to unconstrained spatial navigation. DataDoP~\cite{zhang2025gendop} pairs trajectories with depth maps and text for cinematic applications, yet relies on perspective-view depth covering only a limited field of view.
None of these resources simultaneously provides omnidirectional scene geometry, diverse spatial behaviors, and hierarchical textual guidance. OmniCaT fills this gap by grounding trajectory generation in panoramic point-cloud geometry and pairing each trajectory with structured language annotations across four distinct navigation behaviors.

\noindent\textbf{3D Visual Grounding.}
ReasonGrounder~\cite{liu2025reasongrounder} and SpatialReasoner~\cite{liu2025neural} combine language-guided reasoning with 3D scene representations to localize referred objects. OmniCam connects this grounding problem to camera motion by conditioning trajectory generation on explicit 3D target anchors.

\noindent\textbf{Pose Representation Learning.}
Rotation representation~\cite{kim2023rotation,kazhdan2003rotation} has a well-studied impact on learning stability. Zhou et al.~\cite{zhou2019continuity} show that quaternions and Euler angles suffer from discontinuities under the topology of $\mathrm{SO}(3)$ and propose continuous 6D representations for single-frame rotation regression.
Subsequent work considers sequential settings such as human motion prediction~\cite{cui2020learning,cao2020long} and robotic manipulation~\cite{kim2024openvla,liu2024rdt}. OmniCam uses sequential quaternion sign alignment to reduce inconsistent discrete labels for equivalent rotations, combined with relative translation. This convention does not provide a globally continuous rotation parameterization or an injective finite quantizer.

\section{The OmniCaT Dataset}
Collecting camera trajectories with controlled behavior and target coverage is costly. \textbf{OmniCaT} synthesizes trajectories on geometry reconstructed from panoramic images, enabling controlled variation of spatial behaviors. Feasibility is assessed against the reconstructed geometry and does not by itself guarantee collision avoidance in the physical scene.

\subsection{Dataset Construction}
As illustrated in \Cref{fig:omnicat}, the pipeline consists of three phases. \emph{(1)~Geometric and semantic grounding}: panoramic depth maps are aligned into a global point cloud $\mathbf{P}^{pan}$~\cite{wang2025moge}, filtered for outliers, and augmented with semantic landmarks and 3D volumetric masks using language-assisted annotation and segmentation models~\cite{carion2025sam3segmentconcepts,liu2024grounding,kim2025zim}; a navigation mesh is then constructed to define the traversable space. \emph{(2)~Spatial trajectory synthesis}: four categories of camera behaviors (Target, Surrounding, Reconstruction, and Wandering) are generated on the navigation mesh, with waypoints and orientations refined via spline and spherical linear interpolation. \emph{(3)~Rendering and multimodal annotation}: a point-based renderer produces video sequences from the planned poses, and a hierarchical annotation scheme derives kinematic motion primitives from camera extrinsics before a vision-language model synthesizes motion-aware textual instructions. Full implementation details, including depth alignment optimization, NavMesh engineering, and trajectory planning heuristics, are provided in \Cref{sub:implementation_details}.

\subsection{Dataset Statistics}
\label{sec:Statistics}
\noindent\textbf{Scale and Annotations.} As summarized in \Cref{fig:omnicat}, OmniCaT comprises {267,700} trajectories across {9,998} scenes ({21.7M} rendered frames), categorized into four spatial behaviors: Surrounding ({97,702}), Target ({90,065}), Reconstruction ({43,316}), and Wandering ({36,617}). Each trajectory is annotated with 3D semantic targets, directional bearings, and hierarchical textual descriptions, enabling joint geometric-semantic conditioning. Full annotation details are provided in \Cref{sub:implementation_details}.

\begin{figure*}[tp]
\centering
\includegraphics[width=1.0\linewidth]{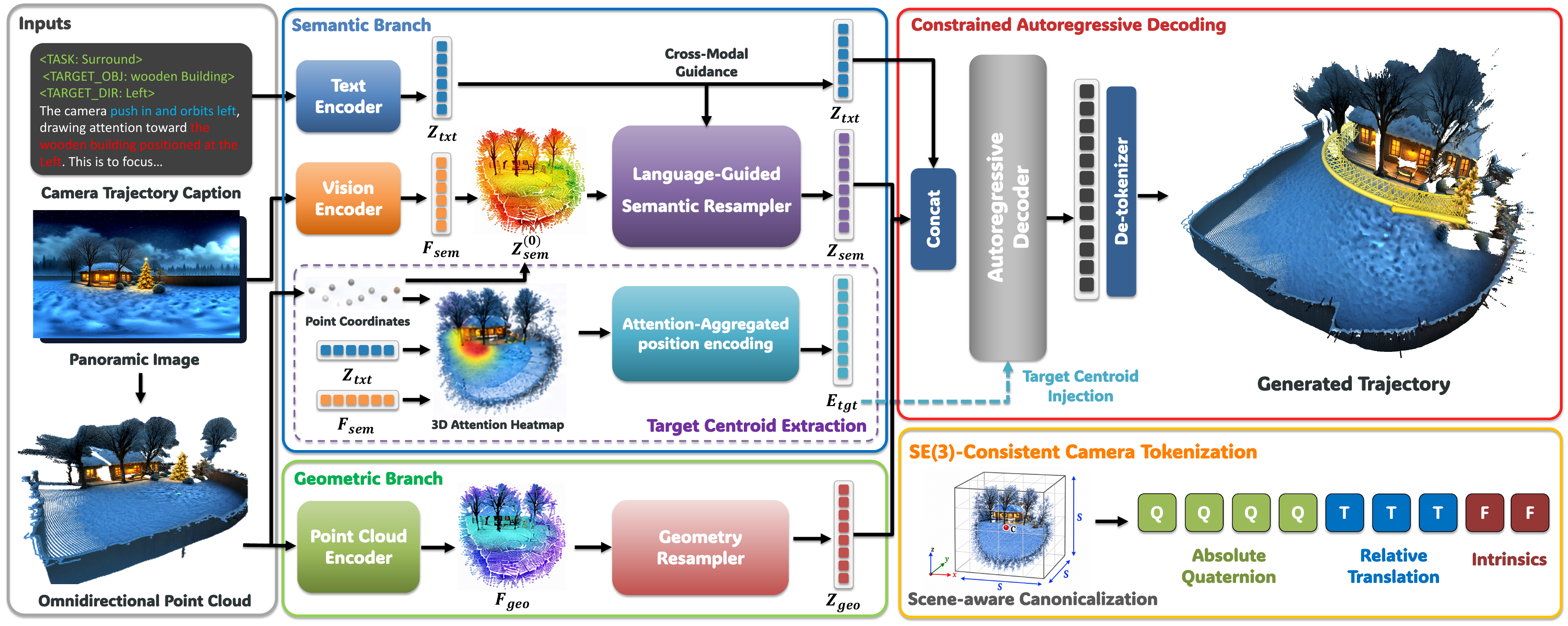}
\caption{\textbf{OmniCam Architecture.} Three parallel branches encode the caption, panoramic image, and point cloud. A Language-Guided Resampler produces target-aware semantic tokens while a Geometry Resampler yields geometric descriptors; an attention-aggregated 3D target centroid provides an explicit spatial anchor for the decoder. The autoregressive decoder generates pose tokens from the concatenated condition, which are mapped to continuous SE(3) trajectories by a de-tokenizer.
Right bottom: SE(3) tokenization with 9 tokens for each pose frame.
}
\label{fig:pipeline}
\end{figure*}

\noindent\textbf{Scope of Synthetic Supervision.}
The planner uses reconstructed geometry, navigation meshes, and semantic targets to generate feasible examples. OmniCam learns to predict trajectories from a panorama and text without constructing a navigation mesh or running graph search at inference. Downstream evaluation tests the usefulness of this learned mapping. A gap to the planner does not establish freedom from planner bias; evaluation on independently measured geometry and independently specified tasks is needed to assess that limitation.

\section{Method}
\label{sec:method}

\subsection{Hybrid Pose Tokenization}
\label{sec:tokenization}
We serialize continuous camera parameters using scene-aware calibration and a hybrid absolute-relative representation. Quaternion sign alignment addresses equivalent positive and negative representations; finite binning still introduces quantization error.

\noindent\textbf{Scene-Aware Calibration.}
Panoramic depth is unprojected into a filtered 3D point cloud. Scene statistics, including a centroid $\mathbf{c}$ and a scale $s$, calibrate camera translations and the initial frame. Normalization reduces variation in coordinate magnitude across scenes but does not eliminate errors in estimated metric scale.

\noindent\textbf{Hybrid Absolute-Relative Pose Encoding.}
Rotations are encoded as absolute unit quaternions with sequential sign alignment: each frame is aligned with the preceding aligned representative. The deterministic convention and its assumptions are stated in Proposition~\ref{prop:hemisphere}; implementation-specific tie handling must agree with that convention. Translations after the initial frame use relative increments, reducing the per-step dynamic range while allowing quantization residuals to accumulate.
These components, together with normalized intrinsics, are quantized via uniform binning into per-frame token groups of 9 tokens: 4 rotation, 3 translation, 2 intrinsics (\Cref{fig:pipeline}, right bottom). The token sequence $\mathbf{S}$ is augmented with boundary tokens (\textbf{BOS}, \textbf{EOS}, \textbf{PAD}) and projected into a latent space through a learnable codebook.

\noindent\textbf{Theoretical Motivation.}
The hybrid scheme balances different reconstruction-error tradeoffs (\Cref{sec:theory_appendix}). If each decoded rotation has angular quantization error at most $\epsilon$, absolute encoding retains this per-frame bound, whereas integrating relative rotations permits a worst-case bound of $\min\{\pi,(T-1)\epsilon\}$ when the initial rotation is exact (Proposition~\ref{prop:rotation_error}). The value of $\epsilon$ depends on the component quantizer and decoder normalization. Relative translations can use a smaller displacement range and hence finer local resolution, but their errors can also accumulate. Proposition~\ref{prop:hemisphere} shows that a specified initial sign rule followed by alignment to the already aligned preceding quaternion removes dependence on the arbitrary quaternion sign when adjacent rotations differ by less than $\pi$. This resolves double-cover ambiguity in the supervision; finite quantization remains lossy and many-to-one.

\subsection{Decoupled Dual-Branch Conditioning}
\label{sec:encoding}
OmniCam encodes estimated panoramic geometry and language-referenced appearance in separate pathways. The point cloud supplies omnidirectional, partial scene structure; the semantic pathway supplies target-related features. This architecture separates input routing, without assuming that the learned representations are statistically independent.

\noindent\textbf{Semantic Branch (Text \& Vision Encoders).} The semantic branch jointly processes the trajectory caption and the panoramic image to produce target-aware features. A pretrained text encoder maps the caption $T_{txt}$ into textual features $\mathbf{Z}_{txt} \in \mathbb{R}^{N_{txt} \times C_{txt}}$, which also serve as cross-modal guidance for the vision pathway. In parallel, a vision encoder produces a visual feature map $\mathbf{F}_{sem} \in \mathbb{R}^{H \times W \times C_{sem}}$ from $\mathbf{I}_{pano}$. Text and the target descriptor $T_{obj}$ guide the subsequent semantic resampling stage. Structured caption dropout varies instruction composition during training (\Cref{sec:caption_dropout}).

\noindent\textbf{Geometric Branch (Point Cloud Encoder).} The geometric branch receives no direct text input. At inference time, a monocular depth estimator~\cite{wang2025moge} reconstructs panoramic depth from the input RGB panorama; the estimated depth is then unprojected into a 3D point cloud, voxelized, and processed by a sparse-convolution encoder combined with a Transformer encoder to yield geometric descriptors $\mathbf{F}_{geo} \in \mathbb{R}^{M \times C_{geo}}$. Its features provide structural context for trajectory generation. Architectural details for both branches are provided in \Cref{sec:encoder_details}.

\subsection{Spatial Grounding via Resampling and Target Centroid Injection}
\label{sec:resampling}

This section introduces two parallel resamplers that compress heterogeneous encoder outputs (text, vision, and point cloud features) into compact prefix tokens, and an attention-aggregated 3D centroid that serves as an explicit spatial anchor.

\noindent\textbf{Language-Guided Semantic Resampler.} Given the visual features from the vision encoder, we first project them into the 3D point-cloud space to obtain 3D semantic features $\mathbf{Z}_{sem}^{(0)}$.
The semantic resampler combines self-attention, spatial attention, and textual attention. The interaction is summarized schematically at layer $l$:
\begin{equation}
\begin{aligned}
\mathbf{Z}_{sem}^{(l)}={}&
\mathrm{SelfAttn}\bigl(\underbrace{\mathbf{Z}_{sem}^{(l-1)}}_{\mathrm{Q,K,V}}\bigr)\\
&+\mathrm{CrossAttn}\bigl(\underbrace{\mathbf{Z}_{sem}^{(l-1)}}_{\mathrm{Q}},\underbrace{[\mathbf{P};\mathbf{F}_{sem}]}_{\mathrm{K,V}}\bigr)\\
&+\mathrm{CrossAttn}\bigl(\underbrace{\mathbf{Z}_{sem}^{(l-1)}}_{\mathrm{Q}},\underbrace{\mathbf{Z}_{txt}}_{\mathrm{K,V}}\bigr).
\end{aligned}
\label{eq:attn_resampler}
\end{equation}
where 3D semantic features serve as the common query; self-attention captures inter-query dependencies; spatial cross-attention grounds queries in physical scene geometry aggregated by 3D point coordinates $\mathbf{P}$ concatenated with visual features $\mathbf{F}_{sem}$, and text-guided cross-attention aggregates textual guidance from $\mathbf{Z}_{txt}$.
The resulting $\mathbf{Z}_{sem} \in \mathbb{R}^{M_{sem} \times C}$ serves as the semantic latent code.

\noindent\textbf{Geometry Resampler.}
A text-agnostic geometry resampler compresses dense per-point descriptors $\mathbf{F}_{geo}$ into a fixed-length latent sequence $\mathbf{Z}_{geo} \in \mathbb{R}^{M_{geo} \times C}$ through a stack of Transformer decoder layers, providing a geometric condition without a direct text input.

\noindent\textbf{Target Centroid Injection.}
A target centroid $\mathbf{c}_{tgt}$ supplies an explicit spatial anchor. Let $a_j$ denote the target-attention logit associated with point $\mathbf{p}_j$ and let $\tau>0$ be the temperature. The described soft-argmax aggregation has the form
\begin{equation}
 \alpha_j=\frac{\exp(a_j/\tau)}{\sum_{k=1}^{N}\exp(a_k/\tau)},\qquad
 \mathbf{c}_{tgt}=\sum_{j=1}^{N}\alpha_j\mathbf{p}_j.
 \label{eq:target_centroid}
\end{equation}
This requires point-indexed weights; attention over text tokens alone does not directly define them.
The centroid is positionally encoded into an auxiliary target memory $\mathbf{E}_{tgt}\in\mathbb{R}^{M_{tgt}\times C}$. When caption dropout removes the target object tag, these tokens are masked.

The final multimodal condition $\mathbf{Z}$ is formed by concatenating the textual features with the resampled panoramic latents:
\begin{equation}
\mathbf{Z} = [\mathbf{Z}_{txt};\; \mathbf{Z}_{sem};\; \mathbf{Z}_{geo}].
\label{eq:z_cond}
\end{equation}
The target token bank $\mathbf{E}_{tgt}$ serves as auxiliary memory accessed by the decoder through a gated cross-attention module, as detailed in \Cref{sec:generation}.

\subsection{Format-Constrained Autoregressive Decoding}
\label{sec:generation}
An autoregressive Transformer decoder $\mathcal{D}$ generates discrete trajectory tokens conditioned on the multimodal prefix $\mathbf{Z}$ and the target token bank $\mathbf{E}_{tgt}$.

\noindent\textbf{Decoder Architecture.} The condition $\mathbf{Z}$ is prepended to the input sequence before the \textbf{BOS} token. Each decoder block consists of causal self-attention, gated cross-attention over $\mathbf{E}_{tgt}$, and a feed-forward network. The gate controls the contribution of the target memory. Its effect on target-directed trajectories is evaluated through the grounding ablations; per-behavior gate strength is not inferred from the aggregate results.

\begin{table*}[tp]
\centering
\caption{\textbf{Comparison with representative trajectory generation baselines on OmniCaT and the public DataDoP benchmark.}\label{tab:main}}
\resizebox{\linewidth}{!}{
\begin{tabular}{l|c|cc|cc|c|cc}
\toprule
\textbf{Method} & \textbf{Dataset} & \textbf{ATE $\downarrow$} & \textbf{FDE $\downarrow$} & \textbf{RPE-R $\downarrow$} &\textbf{ RPE-T $\downarrow$} & \textbf{Coll.~(\%) $\downarrow$} & \textbf{Vis.~(\%) $\uparrow$} & \textbf{Angle~($^{\circ}$) $\downarrow$} \\
\midrule
\multicolumn{9}{c}{\emph{Evaluation on the OmniCaT benchmark}} \\
\midrule
CCD~\cite{jiang2024cinematographic}      & Pre-trained & 2.483 &3.396 & 3.112 & 0.082 & 36.9 & 12.3 & 32.5 \\
E.T.~\cite{courant2024exceptional}       & Pre-trained & 1.983 & 2.368 & 2.151 & 0.069 & 46.2 & 21.3 & 40.2 \\
Director3D~\cite{li2024director3d}       & Pre-trained & 1.828 & 2.942 & 2.309 & 0.052 & 28.4 & 28.1 & 27.2 \\
Director3D       & OmniCaT & 1.522 & 2.648 & 2.187 & 0.049 & 27.6 & 29.3 & 26.5 \\
GenDoP~\cite{zhang2025gendop}            & Pre-trained & 1.865 & 2.951 & 1.681 & 0.044 & 32.6 & 31.5 & 29.1 \\
GenDoP            & OmniCaT & 1.624 & 2.483 & 1.577 & 0.038 & 30.4 & 33.6 & 27.6 \\
\midrule
\cellcolor[HTML]{efefef}\textbf{OmniCam (Ours)}                  & \cellcolor[HTML]{efefef}\textbf{OmniCaT}   & \cellcolor[HTML]{efefef}\textbf{0.868} & \cellcolor[HTML]{efefef}\textbf{1.462} & \cellcolor[HTML]{efefef}\textbf{1.135} & \cellcolor[HTML]{efefef}\textbf{0.027} & \cellcolor[HTML]{efefef}\textbf{10.4} & \cellcolor[HTML]{efefef}\textbf{52.0} & \cellcolor[HTML]{efefef}\textbf{10.9} \\
\midrule
\multicolumn{9}{c}{\emph{Evaluation on the public DataDoP benchmark~\cite{zhang2025gendop}}} \\
\midrule
CCD~\cite{jiang2024cinematographic}      & Pre-trained & 2.697 &3.338 & 3.218 & 0.093 & 44.2 & 18.7 & 34.2 \\
E.T.~\cite{courant2024exceptional}       & Pre-trained & 2.568 & 3.142 & 3.086 & 0.084 & 45.9 & 22.5 & 33.8 \\
Director3D~\cite{li2024director3d}       & Pre-trained & 1.944 & 2.887 & 2.968 & 0.077 & 39.6 & 30.8 & 34.3 \\
Director3D       & DataDoP 		  		        & 1.688 & 2.762 & 2.819 & 0.062 & 41.7 & 32.2 & 29.5 \\
GenDoP~\cite{zhang2025gendop}            & DataDoP & 1.437 & 2.255 & 1.642 & 0.058 & 28.5 & 36.6 & 29.8 \\
\midrule
\cellcolor[HTML]{efefef}\textbf{OmniCam (Ours)}                  & \cellcolor[HTML]{efefef}\textbf{OmniCaT}   & \cellcolor[HTML]{efefef}\textbf{0.925} & \cellcolor[HTML]{efefef}\textbf{1.834} & \cellcolor[HTML]{efefef}\textbf{1.118} & \cellcolor[HTML]{efefef}\textbf{0.046} & \cellcolor[HTML]{efefef}\textbf{8.9} & \cellcolor[HTML]{efefef}\textbf{64.2} & \cellcolor[HTML]{efefef}\textbf{12.7} \\
\bottomrule
\end{tabular}
}
\end{table*}

\begin{figure*}[!htbp]
    \centering
    \includegraphics[width=1.0\linewidth]{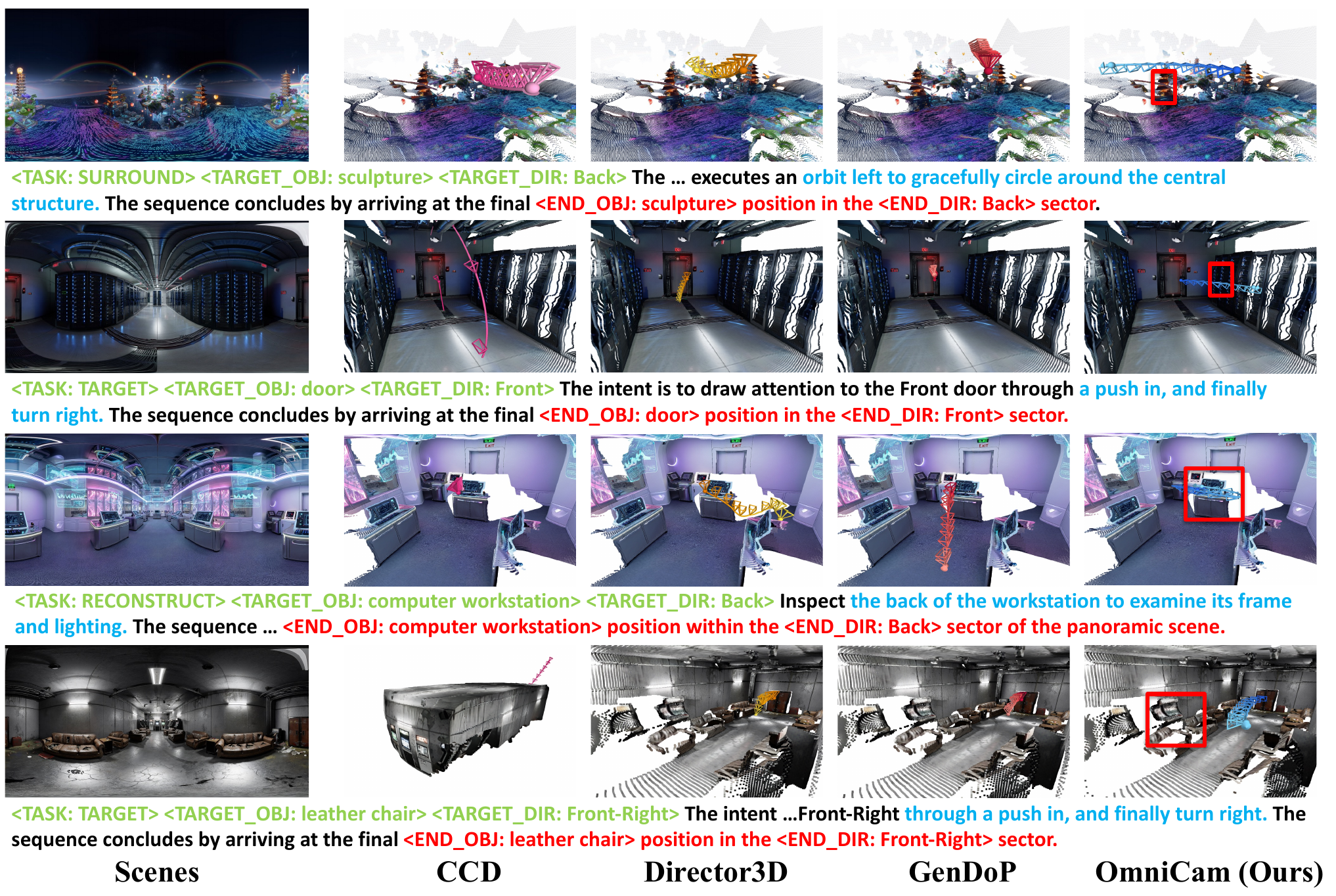}
    \caption{\textbf{Qualitative comparison on representative panoramic scenes.} Each column shows the trajectory produced by a specific method. The target object is indicated by a red bounding box.}
    \label{fig:qualitative}
\end{figure*}

\noindent\textbf{Format-Constrained Decoding.}
During inference, the \textbf{EOS} token is masked within each group of 9 tokens (rotation, translation, and intrinsics), so termination occurs only at a frame boundary. The initial frame is constrained by the calibrated origin. This is a sequence-format constraint: valid rotations require an appropriate quaternion decoding rule, and collision avoidance is assessed empirically rather than guaranteed by token masking. Top-$k$ sampling controls the sampling support.

\noindent\textbf{Training Objective.} The model is optimized with a weighted cross-entropy loss $\mathcal{L}_{ce}$ over pose tokens, where rotation and translation channels receive higher weights to prioritize geometric accuracy. An auxiliary smooth $\ell_1$ localization loss $\mathcal{L}_{loc}$ supervises the estimated centroid $\mathbf{c}_{tgt}$ against the annotated 3D target position: $\mathcal{L} = \mathcal{L}_{ce} + \lambda \cdot \mathcal{L}_{loc}$.

\section{Experiments}
\label{sec:experiments}

\begin{figure*}[t]
    \centering
    \includegraphics[width=1.0\linewidth]{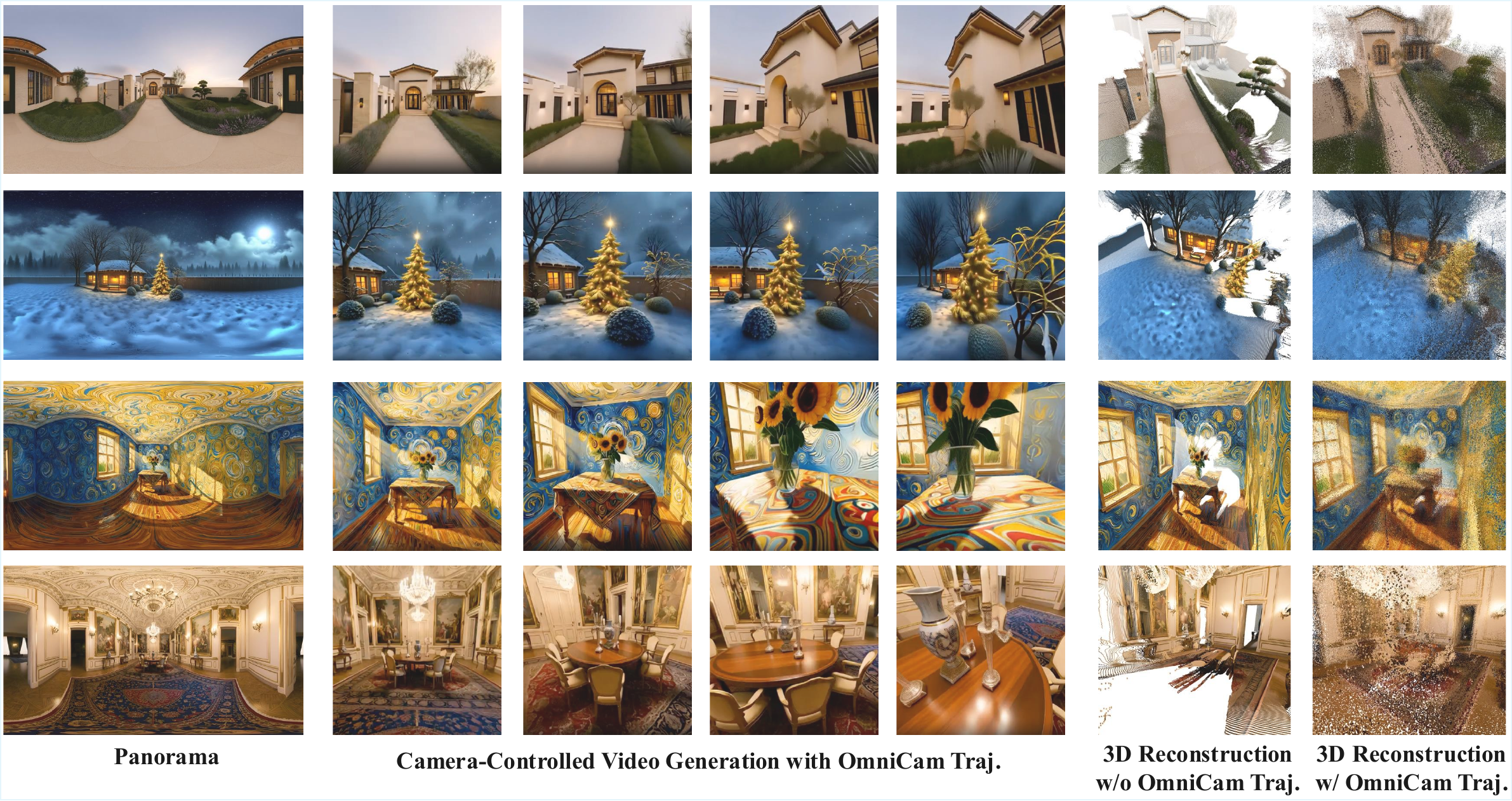}
    \caption{\textbf{Qualitative comparison on panoramic scenes.} Each row displays an input panorama, generated video sequences, and 3D reconstruction results. The two rightmost columns compare reconstruction outputs without and with the OmniCam trajectory; no quantitative reconstruction metric is reported here.}
    \label{fig:qualitative_recon}
\end{figure*}

\begin{figure*}[t]
\centering
\includegraphics[width=\linewidth]{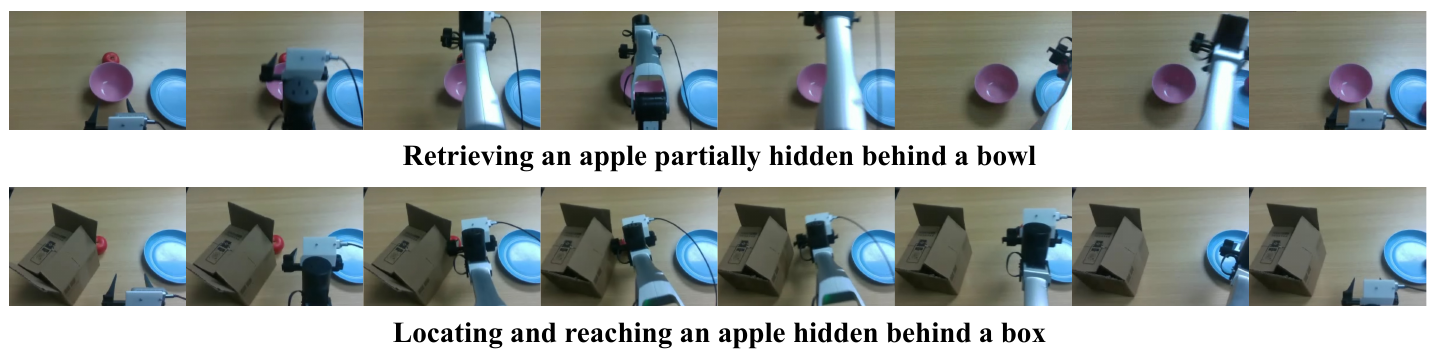}
\caption{\textbf{Real-robot demonstration of OmniCam-guided active perception.} OmniCam plans a camera trajectory that actively repositions the observation camera to resolve target occlusion. }
\label{fig:real_robot}
\end{figure*}

\subsection{Experimental Setup}
\noindent\textbf{Datasets.} OmniCaT experiments are reported on disjoint training and evaluation splits. Exact scene-level counts and split manifests require verification. DataDoP~\cite{zhang2025gendop} results are reported for OmniCam trained on OmniCaT without DataDoP fine-tuning. The construction of OmniCam inputs and the evaluation geometry on DataDoP must be specified before interpreting these as a controlled zero-shot comparison.

\noindent\textbf{Baselines.} Director3D~\cite{li2024director3d} and GenDoP~\cite{zhang2025gendop} are retrained on OmniCaT. CCD~\cite{jiang2024cinematographic} and E.T.~\cite{courant2024exceptional} use released checkpoints as out-of-domain references. Their character-centric conditioning differs from the available OmniCaT inputs, so these rows do not isolate model quality under matched conditioning.

\noindent\textbf{Metrics.} We report trajectory errors (ATE, FDE, and RPE-R/T, using the \textsc{evo} toolkit~\cite{grupp2017evo} where applicable), collision against reconstructed navigation meshes, target visibility, and angular deviation. These reference-trajectory errors measure agreement with the generated supervision; multiple distinct trajectories can satisfy the same instruction.

\noindent\textbf{Implementation.} OmniCam uses an OPT-style~\cite{zhang2022opt} decoder with FlashAttention-2~\cite{dao2023flashattention}, trained in bf16 on 8$\times$H20 GPUs for $\sim$100 epochs. The base configuration (hidden dim 1024, 8 heads, 12 layers; ${\sim}$500M trainable parameters) is used for the main comparisons; a larger variant (${\sim}$1B) is analyzed in \Cref{sec:model_scaling}. Hyperparameters are provided in \Cref{sec:impl_details}.

\begin{table*}[tp]
\centering
\small
\caption{\textbf{Ablation studies on OmniCaT.} The dual-branch panoramic encoder, target injection components, and pose tokenization schemes are investigated.}
\label{tab:merged_ablation}
\resizebox{\linewidth}{!}{
\begin{tabular}{ll|cc|cc|c|cc}
\toprule
& \textbf{Variant} & \textbf{ATE $\downarrow$} & \textbf{FDE $\downarrow$} & \textbf{RPE-R $\downarrow$} & \textbf{RPE-T $\downarrow$} & \textbf{Coll.~(\%) $\downarrow$} & \textbf{Vis.~(\%) $\uparrow$} & \textbf{Angle~($^{\circ}$) $\downarrow$} \\
\midrule
\multirow{2}{*}{\textit{Encoder}} & w/o semantic branch     & 2.186 & 3.247 & 2.694 & 0.072 & 13.6 & 11.5 & 64.7 \\
& w/o geometric branch    & 1.774 & 3.084 & 2.881 & 0.049 & 45.8 & 33.6 & 19.7 \\
\midrule
\multirow{2}{*}{\textit{Geo.\ Input}} & Panorama RGB only           & 1.536 & 2.687 & 2.143 & 0.044 & 38.7 & 36.2 & 18.4 \\
& Panorama RGB-D              & 1.247 & 2.118 & 1.685 & 0.036 & 22.6 & 42.8 & 14.7 \\
\midrule
\multirow{3}{*}{\textit{Grounding}} & w/o target cross-attention          & 1.673 & 2.259 & 2.774 & 0.042 & 15.2 & 25.8 & 35.6 \\
& w/o target localization loss        & 1.582 & 2.148 & 2.661  & 0.038 & 14.8 & 32.4 & 28.7 \\
& w/o target cross-attn. \& loss      & 1.924 & 2.576 & 2.903 & 0.051 & 18.9 & 21.3 & 40.2 \\
\midrule
\multirow{3}{*}{\textit{Tokenization}} & All absolute (rot. + trans.)            & 1.278 & 1.886 & 1.625 & 0.058 & 14.3 & 41.6 & 15.3 \\
& All relative (rot. + trans.)          &1.449 & 1.927 & 1.984 & 0.062 & 16.6 & 43.2 & 16.4 \\
& Hybrid w/o hemisphere consistency   & 1.094 & 1.756 & 1.847 & 0.033 & 12.2 & 46.1 & 13.8 \\
\midrule
\cellcolor[HTML]{efefef}\textbf{\textit{Full}} & \cellcolor[HTML]{efefef}\textbf{OmniCam}      & \cellcolor[HTML]{efefef}\textbf{0.868} & \cellcolor[HTML]{efefef}\textbf{1.462} & \cellcolor[HTML]{efefef}\textbf{1.135} & \cellcolor[HTML]{efefef}\textbf{0.027} & \cellcolor[HTML]{efefef}\textbf{10.4} & \cellcolor[HTML]{efefef}\textbf{52.0} & \cellcolor[HTML]{efefef}\textbf{10.9} \\
\bottomrule
\end{tabular}
}
\end{table*}

\subsection{Main Results}
\label{sec:main_results}

\noindent\textbf{Quantitative Comparison.} On OmniCaT (\Cref{tab:main}), OmniCam reduces ATE, FDE, RPE-R, and RPE-T by 46.6\%, 41.1\%, 28.0\%, and 28.9\% relative to retrained GenDoP, respectively. Collision decreases from 30.4\% to 10.4\% (65.8\% relative reduction), and visibility rises from 33.6\% to 52.0\% (54.8\% relative increase). Director3D is the best baseline for ATE and collision; relative to it, the reductions are 43.0\% and 62.3\%. The scene-bootstrap intervals for OmniCam and GenDoP are reported in \Cref{sec:confidence}; component-specific effects are examined separately below.

\noindent\textbf{Zero-Shot Cross-Dataset Transfer.}
The reported DataDoP results favor OmniCam over GenDoP trained on DataDoP: ATE 0.925 vs.\ 1.437, collision 8.9\% vs.\ 28.5\%, visibility 64.2\% vs.\ 36.6\%, and angle 12.7$^{\circ}$ vs.\ 29.8$^{\circ}$. OmniCam is reported to use no DataDoP fine-tuning. Establishing the extent of transfer requires the input adaptation, data-overlap checks, and metric protocol identified above; these aggregate values alone do not rule out planner bias or differences in available scene information.

\noindent\textbf{Comparison with the Heuristic Planner.}
With access to reconstructed geometry, navigation meshes, and semantic annotations, the planner has 5.2\% collision and 58.7\% visibility, compared with OmniCam's 10.4\% and 52.0\%. This is a privileged-information reference, not a mathematical performance bound. The reported runtimes are 0.84s and approximately 16.5s (\Cref{sec:planner_oracle}); matched timing conditions require verification.

\noindent\textbf{Per-Behavior Analysis.} \Cref{sec:per_behavior} reports lower ATE and FDE for OmniCam in all four behavior categories. These results describe coverage of the evaluated behaviors without identifying which component causes each gain.

\noindent\textbf{Qualitative Comparison.} \Cref{fig:qualitative} compares trajectories across representative scenes. OmniCam trajectories remain aligned with navigable regions, follow the intended viewpoint of the referenced object, and exhibit smoother transitions in both position and orientation compared to all baselines. Additional qualitative examples are provided in \Cref{sec:more_qualitative}.

\subsection{Ablation Studies}
\label{sec:ablation}

\noindent\textbf{Dual-Branch Encoder.} Removing the semantic branch reduces visibility from 52.0\% to 11.5\%, while removing the geometric branch increases collision from 10.4\% to 45.8\%. This supports the usefulness of both pathways. Because branch removal also changes capacity and available information, a capacity-matched single-branch comparison is needed to isolate the effect of routing.

\noindent\textbf{Geometric Input Modality.} RGB-only, RGB-D, and point-cloud variants yield collision rates of 38.7\%, 22.6\%, and 10.4\%, respectively. The results favor explicit 3D processing in these configurations. Since the point cloud is unprojected from depth, this comparison concerns representation and processing, rather than additional geometric information created by unprojection.

\noindent\textbf{Target Centroid Injection.} Removing target cross-attention changes visibility from 52.0\% to 25.8\% and angle from 10.9$^{\circ}$ to 35.6$^{\circ}$. Removing the localization loss yields 32.4\% and 28.7$^{\circ}$; removing both yields 21.3\% and 40.2$^{\circ}$, the worst values among these grounding variants. The comparison supports explicit target anchoring in the tested architecture (\Cref{sec:attn_vis}).

\noindent\textbf{Hybrid Pose Encoding.} Hybrid encoding obtains RPE-R 1.135 and RPE-T 0.027, compared with 1.625/0.058 for all-absolute and 1.984/0.062 for all-relative encoding. Removing sign alignment increases RPE-R to 1.847. The representation analysis in \Cref{sec:theory_appendix} concerns quantization and sign selection; it does not guarantee the accuracy of autoregressive predictions.

\subsection{Downstream Application Validation}
\label{sec:downstream}

With WorldStereo~\cite{zhang2026worldstereo}, OmniCam trajectories improve the reported CLIP score by 14.1\% and consecutive-frame SSIM by 15.4\% relative to GenDoP. \Cref{fig:qualitative_recon} illustrates video and reconstruction outputs.

Embodied policies combine perception with action generation, as in TriVLA~\cite{liu2025trivla} and spatiotemporal diffusion policies~\cite{liu2025spatial}, while ActiveVLA~\cite{liu2026activevla} explicitly selects task-relevant views for manipulation. OmniCam addresses the complementary problem of generating camera trajectories for acquiring such observations. Robotic active-perception results and a qualitative real-robot demonstration are presented in \Cref{sec:downstream_appendix,fig:real_robot}. The robot success-rate aggregation needs verification; the demonstration is evidence of feasibility rather than a systematic sim-to-real evaluation.

\section{Conclusion}
OmniCam generates camera trajectories from panoramic imagery and language using point-cloud conditioning, hybrid pose serialization, and explicit 3D target anchoring. OmniCaT supplies 267,700 examples across four camera behaviors. The reported comparisons and ablations support these design choices within the evaluated settings, and downstream studies examine their practical use. The representation analysis distinguishes quaternion sign consistency from lossy quantization and identifies the error-accumulation tradeoff of relative encoding. Remaining limitations include estimated and incomplete geometry, planner-derived supervision, static-scene observations, and the need for independently verified evaluation protocols. Future work includes dynamic scenes and closed-loop planning.

\FloatBarrier
{
\small
\bibliographystyle{ieeenat_fullname}
\bibliography{reference}

@String(ICCV  = {Int. Conf. Comput. Vis.})

@String(TOG   = {ACM Trans. Graph.})

@String(ICCV  = {ICCV})

@String(TOG   = {ACM TOG})

@inproceedings{yu2023mvimgnet,
  title={{MVImgNet}: A large-scale dataset of multi-view images},
  author={Yu, Xianggang and Xu, Mutian and Zhang, Yidan and Liu, Haolin and Ye, Chongjie and Wu, Yushuang and Yan, Zizheng and Zhu, Chenming and Xiong, Zhangyang and Liang, Tianyou and others},
  booktitle={Proceedings of the IEEE/CVF conference on computer vision and pattern recognition},
  pages={9150--9161},
  year={2023}
}

@article{zhou2018stereo,
  title={Stereo magnification: Learning view synthesis using multiplane images},
  author={Zhou, Tinghui and Tucker, Richard and Flynn, John and Fyffe, Graham and Snavely, Noah},
  journal={arXiv preprint arXiv:1805.09817},
  year={2018}
}

@article{he2024cameractrl,
  title={{CameraCtrl}: Enabling camera control for text-to-video generation},
  author={He, Hao and Xu, Yinghao and Guo, Yuwei and Wetzstein, Gordon and Dai, Bo and Li, Hongsheng and Yang, Ceyuan},
  journal={arXiv preprint arXiv:2404.02101},
  year={2024}
}

@article{intelligence2025pi,
  title={$\pi_{0.5}$: a Vision-Language-Action Model with Open-World Generalization},
  author={Intelligence, Physical and Black, Kevin and Brown, Noah and Darpinian, James and Dhabalia, Karan and Driess, Danny and Esmail, Adnan and Equi, Michael and Finn, Chelsea and Fusai, Niccolo and others},
  journal={arXiv preprint arXiv:2504.16054},
  year={2025}
}

@article{kim2024openvla,
  title={{OpenVLA}: An open-source vision-language-action model},
  author={Kim, Moo Jin and Pertsch, Karl and Karamcheti, Siddharth and Xiao, Ted and Balakrishna, Ashwin and Nair, Suraj and Rafailov, Rafael and Foster, Ethan and Lam, Grace and Sanketi, Pannag and others},
  journal={arXiv preprint arXiv:2406.09246},
  year={2024}
}

@inproceedings{ling2024dl3dv,
  title={{DL3DV-10K}: A large-scale scene dataset for deep learning-based {3D} vision},
  author={Ling, Lu and Sheng, Yichen and Tu, Zhi and Zhao, Wentian and Xin, Cheng and Wan, Kun and Yu, Lantao and Guo, Qianyu and Yu, Zixun and Lu, Yawen and others},
  booktitle={Proceedings of the IEEE/CVF Conference on Computer Vision and Pattern Recognition},
  pages={22160--22169},
  year={2024}
}

@article{jiang2024cinematographic,
  title={Cinematographic camera diffusion model},
  author={Jiang, Hongda and Wang, Xi and Christie, Marc and Liu, Libin and Chen, Baoquan},
  journal={Computer Graphics Forum},
  volume={43},
  number={2},
  pages={e15055},
  year={2024},
  doi={10.1111/cgf.15055},
  publisher={Wiley}
}

@article{alonso2024diffusion,
  title={Diffusion for world modeling: Visual details matter in {Atari}},
  author={Alonso, Eloi and Jelley, Adam and Micheli, Vincent and Kanervisto, Anssi and Storkey, Amos and Pearce, Tim and Fleuret, Fran{\c{c}}ois},
  journal={Advances in Neural Information Processing Systems},
  volume={37},
  pages={58757--58791},
  year={2024}
}

@inproceedings{courant2024exceptional,
  title={{ET} the Exceptional Trajectories: Text-to-camera-trajectory generation with character awareness},
  author={Courant, Robin and Dufour, Nicolas and Wang, Xi and Christie, Marc and Kalogeiton, Vicky},
  booktitle={European Conference on Computer Vision},
  pages={464--480},
  year={2024},
  organization={Springer}
}

@article{li2024director3d,
  title={{Director3D}: Real-world camera trajectory and {3D} scene generation from text},
  author={Li, Xinyang and Lai, Zhangyu and Xu, Linning and Qu, Yansong and Cao, Liujuan and Zhang, Shengchuan and Dai, Bo and Ji, Rongrong},
  journal={Advances in neural information processing systems},
  volume={37},
  pages={75125--75151},
  year={2024}
}

@inproceedings{zhang2025gendop,
  title={{GenDoP}: Auto-regressive camera trajectory generation as a director of photography},
  author={Zhang, Mengchen and Wu, Tong and Tan, Jing and Liu, Ziwei and Wetzstein, Gordon and Lin, Dahua},
  booktitle={Proceedings of the IEEE/CVF International Conference on Computer Vision},
  pages={18229--18239},
  year={2025}
}

@article{blinn2002looking,
  title={Where am I? What am I looking at?(cinematography)},
  author={Blinn, Jim},
  journal={IEEE Computer Graphics and Applications},
  volume={8},
  number={4},
  pages={76--81},
  year={2002},
  publisher={IEEE}
}

@article{putz20163d,
  title={{3D} navigation mesh generation for path planning in uneven terrain},
  author={P{\"u}tz, Sebastian and Wiemann, Thomas and Sprickerhof, Jochen and Hertzberg, Joachim},
  journal={Ifac-Papersonline},
  volume={49},
  number={15},
  pages={212--217},
  year={2016},
  publisher={Elsevier}
}

@misc{grupp2017evo,
  title={{evo}: Python package for the evaluation of odometry and {SLAM}},
  author={Grupp, Michael},
  howpublished={\url{https://github.com/MichaelGrupp/evo}},
  year={2017}
}

@inproceedings{ruetz2019ovpc,
  title={{OVPC} mesh: {3D} free-space representation for local ground vehicle navigation},
  author={Ruetz, Fabio and Hern{\'a}ndez, Emili and Pfeiffer, Mark and Oleynikova, Helen and Cox, Mark and Lowe, Thomas and Borges, Paulo},
  booktitle={2019 International Conference on Robotics and Automation (ICRA)},
  pages={8648--8654},
  year={2019},
  organization={IEEE}
}

@article{lino2015intuitive,
  title={Intuitive and efficient camera control with the toric space},
  author={Lino, Christophe and Christie, Marc},
  journal={ACM Transactions on Graphics (TOG)},
  volume={34},
  number={4},
  pages={1--12},
  year={2015},
  publisher={ACM New York, NY, USA}
}

@article{zhang2022opt,
  title={{OPT}: Open pre-trained transformer language models},
  author={Zhang, Susan and Roller, Stephen and Goyal, Naman and Artetxe, Mikel and Chen, Moya and Chen, Shuohui and Dewan, Christopher and Diab, Mona and Li, Xian and Lin, Xi Victoria and others},
  journal={arXiv preprint arXiv:2205.01068},
  year={2022}
}

@article{dao2023flashattention,
  title={{FlashAttention-2}: Faster attention with better parallelism and work partitioning},
  author={Dao, Tri},
  journal={arXiv preprint arXiv:2307.08691},
  year={2023}
}

@inproceedings{galvane2015camera,
  title={Camera-on-rails: automated computation of constrained camera paths},
  author={Galvane, Quentin and Christie, Marc and Lino, Chrsitophe and Ronfard, R{\'e}mi},
  booktitle={Proceedings of the 8th ACM SIGGRAPH Conference on Motion in Games},
  pages={151--157},
  year={2015}
}

@article{liu2024splatraj,
  title={{SplaTraj}: Camera trajectory generation with semantic gaussian splatting},
  author={Liu, Xinyi and Zhang, Tianyi and Johnson-Roberson, Matthew and Zhi, Weiming},
  journal={arXiv preprint arXiv:2410.06014},
  year={2024}
}

@article{bonatti2020autonomous,
  title={Autonomous aerial cinematography in unstructured environments with learned artistic decision-making},
  author={Bonatti, Rogerio and Wang, Wenshan and Ho, Cherie and Ahuja, Aayush and Gschwindt, Mirko and Camci, Efe and Kayacan, Erdal and Choudhury, Sanjiban and Scherer, Sebastian},
  journal={Journal of Field Robotics},
  volume={37},
  number={4},
  pages={606--641},
  year={2020},
  publisher={Wiley Online Library}
}

@article{wan2025wan,
  title={{Wan}: Open and advanced large-scale video generative models},
  author={Wan, Team and Wang, Ang and Ai, Baole and Wen, Bin and Mao, Chaojie and Xie, Chen-Wei and Chen, Di and Yu, Feiwu and Zhao, Haiming and Yang, Jianxiao and others},
  journal={arXiv preprint arXiv:2503.20314},
  year={2025}
}

@article{kim2023rotation,
  title={Rotation representations and their conversions},
  author={Kim, Soohwan and Kim, Minkyoung},
  journal={Ieee Access},
  volume={11},
  pages={6682--6699},
  year={2023},
  publisher={IEEE}
}

@inproceedings{kazhdan2003rotation,
  title={Rotation invariant spherical harmonic representation of 3 d shape descriptors},
  author={Kazhdan, Michael and Funkhouser, Thomas and Rusinkiewicz, Szymon},
  booktitle={Symposium on geometry processing},
  volume={6},
  pages={156--164},
  year={2003}
}

@inproceedings{zhou2019continuity,
  title={On the continuity of rotation representations in neural networks},
  author={Zhou, Yi and Barnes, Connelly and Lu, Jingwan and Yang, Jimei and Li, Hao},
  booktitle={Proceedings of the IEEE/CVF conference on computer vision and pattern recognition},
  pages={5745--5753},
  year={2019}
}

@inproceedings{drucker1992cinema,
  title={Cinema: A system for procedural camera movements},
  author={Drucker, Steven M and Galyean, Tinsley A and Zeltzer, David},
  booktitle={Proceedings of the 1992 symposium on Interactive 3D graphics},
  pages={67--70},
  year={1992}
}

@inproceedings{huang2019learning,
  title={Learning to film from professional human motion videos},
  author={Huang, Chong and Lin, Chuan-En and Yang, Zhenyu and Kong, Yan and Chen, Peng and Yang, Xin and Cheng, Kwang-Ting},
  booktitle={Proceedings of the IEEE/CVF Conference on Computer Vision and Pattern Recognition},
  pages={4244--4253},
  year={2019}
}

@article{jiang2020example,
  title={Example-driven virtual cinematography by learning camera behaviors.},
  author={Jiang, Hongda and Wang, Bin and Wang, Xi and Christie, Marc and Chen, Baoquan},
  journal={ACM Trans. Graph.},
  volume={39},
  number={4},
  pages={45},
  year={2020}
}

@inproceedings{bar2025navigation,
  title={Navigation world models},
  author={Bar, Amir and Zhou, Gaoyue and Tran, Danny and Darrell, Trevor and LeCun, Yann},
  booktitle={Proceedings of the Computer Vision and Pattern Recognition Conference},
  pages={15791--15801},
  year={2025}
}

@article{hou2024learning,
  title={Learning camera movement control from real-world drone videos},
  author={Hou, Yunzhong and Zheng, Liang and Torr, Philip},
  journal={arXiv preprint arXiv:2412.09620},
  year={2024}
}

@inproceedings{cui2020learning,
  title={Learning dynamic relationships for {3D} human motion prediction},
  author={Cui, Qiongjie and Sun, Huaijiang and Yang, Fei},
  booktitle={Proceedings of the IEEE/CVF conference on computer vision and pattern recognition},
  pages={6519--6527},
  year={2020}
}

@inproceedings{cao2020long,
  title={Long-term human motion prediction with scene context},
  author={Cao, Zhe and Gao, Hang and Mangalam, Karttikeya and Cai, Qi-Zhi and Vo, Minh and Malik, Jitendra},
  booktitle={European Conference on Computer Vision},
  pages={387--404},
  year={2020},
  organization={Springer}
}

@article{wang2025moge,
  title={{MoGe-2}: Accurate monocular geometry with metric scale and sharp details},
  author={Wang, Ruicheng and Xu, Sicheng and Dong, Yue and Deng, Yu and Xiang, Jianfeng and Lv, Zelong and Sun, Guangzhong and Tong, Xin and Yang, Jiaolong},
  journal={arXiv preprint arXiv:2507.02546},
  year={2025}
}

@inproceedings{liu2024grounding,
  title={{Grounding DINO}: Marrying {DINO} with grounded pre-training for open-set object detection},
  author={Liu, Shilong and Zeng, Zhaoyang and Ren, Tianhe and Li, Feng and Zhang, Hao and Yang, Jie and Jiang, Qing and Li, Chunyuan and Yang, Jianwei and Su, Hang and others},
  booktitle={European conference on computer vision},
  pages={38--55},
  year={2024},
  organization={Springer}
}

@article{liu2024rdt,
  title={{RDT-1B}: a diffusion foundation model for bimanual manipulation},
  author={Liu, Songming and Wu, Lingxuan and Li, Bangguo and Tan, Hengkai and Chen, Huayu and Wang, Zhengyi and Xu, Ke and Su, Hang and Zhu, Jun},
  journal={arXiv preprint arXiv:2410.07864},
  year={2024}
}

@inproceedings{kim2025zim,
  title={{ZIM}: Zero-shot image matting for anything},
  author={Kim, Beomyoung and Shin, Chanyong and Jeong, Joonhyun and Jung, Hyungsik and Lee, Se-Yun and Chun, Sewhan and Hwang, Dong-Hyun and Yu, Joonsang},
  booktitle={Proceedings of the IEEE/CVF International Conference on Computer Vision},
  pages={23828--23838},
  year={2025}
}

@misc{carion2025sam3segmentconcepts,
  title={{SAM 3}: Segment anything with concepts},
  author={Carion, Nicolas and Gustafson, Laura and Hu, Yuan-Ting and Debnath, Shoubhik and Hu, Ronghang and Suris, Didac and Ryali, Chaitanya and Alwala, Kalyan Vasudev and Khedr, Haitham and Huang, Andrew and others},
  howpublished={arXiv preprint arXiv:2511.16719},
  eprint={2511.16719},
  archivePrefix={arXiv},
  primaryClass={cs.CV},
  url={https://arxiv.org/abs/2511.16719},
  year={2025}
}

@article{zhang2026worldstereo,
  title={{WorldStereo}: Bridging Camera-Guided Video Generation and Scene Reconstruction via {3D} Geometric Memories},
  author={Zhang, Yisu and Cao, Chenjie and Wang, Tengfei and Zuo, Xuhui and Wu, Junta and Zhu, Jianke and Guo, Chunchao},
  journal={arXiv preprint arXiv:2603.02049},
  year={2026}
}

@article{liu2025reasongrounder,
  title={Reasongrounder: Lvlm-guided hierarchical feature splatting for open-vocabulary 3d visual grounding and reasoning},
  author={Liu, Zhenyang and Wang, Yikai and Zheng, Sixiao and Pan, Tongying and Liang, Longfei and Fu, Yanwei and Xue, Xiangyang},
  journal={arXiv preprint arXiv:2503.23297},
  year={2025}
}

@article{liu2026activevla,
  title={ActiveVLA: Injecting active perception into vision-language-action models for precise 3D robotic manipulation},
  author={Liu, Zhenyang and Gu, Yongchong and Wang, Yikai and Xue, Xiangyang and Fu, Yanwei},
  journal={arXiv preprint arXiv:2601.08325},
  year={2026}
}

@article{liu2025trivla,
  title={Trivla: A triple-system-based unified vision-language-action model for general robot control},
  author={Liu, Zhenyang and Gu, Yongchong and Zheng, Sixiao and Xue, Xiangyang and Fu, Yanwei},
  journal={arXiv e-prints},
  pages={arXiv--2507},
  year={2025}
}

@inproceedings{liu2025neural,
  title={A neural representation framework with llm-driven spatial reasoning for open-vocabulary 3d visual grounding},
  author={Liu, Zhenyang and Zheng, Sixiao and Chen, Siyu and Zhao, Cairong and Liang, Longfei and Xue, Xiangyang and Fu, Yanwei},
  booktitle={Proceedings of the 33rd ACM International Conference on Multimedia},
  pages={1042--1051},
  year={2025}
}

@inproceedings{liu2025spatial,
  title={Spatial-temporal aware visuomotor diffusion policy learning},
  author={Liu, Zhenyang and Wang, Yikai and Wang, Kuanning and Liang, Longfei and Xue, Xiangyang and Fu, Yanwei},
  booktitle={2025 IEEE/CVF International Conference on Computer Vision (ICCV)},
  pages={1--10},
  year={2025},
  organization={IEEE}
}
}

\clearpage
\appendix
\section{Implementation Details of OmniCaT}
\label{sub:implementation_details}

\subsection{Geometric Initialization and Refinement}
\noindent\textbf{Depth Alignment Optimization.}
During geometry initialization, the ERP space is subdivided into 42 perspective views (an increase from the default 12) to ensure dense coverage. Depth maps are aligned via a GPU-accelerated Least-Squares Minimal Residual (LSMR) solver. To mitigate artifacts, a grounding pipeline is employed to mask sky regions and to remove depth discontinuities (edge floaters).

\noindent\textbf{NavMesh Engineering.}
The raw NavMesh generated by the Recast algorithm is configured with cell size 0.1\,m, cell height 0.1\,m, agent height 0.1\,m, agent radius 0.05\,m, maximum climb 0.05\,m, and maximum slope 30$^{\circ}$. Post-processing consists of: (1)~height filtering to remove polygons above a roof threshold; (2)~connectivity filtering to retain only the largest connected component; (3)~surface snapping through dense ray-casting to correct misaligned vertices that lie below the actual mesh surface; and (4)~boundary erosion using KD-Tree accelerated searches to maintain a safety buffer from mesh edges. The resulting navigation graph is constructed by sampling points on the NavMesh surface at 0.05\,m spacing (up to 600K nodes), connecting neighbors within $5\times$ spacing with a height coherence constraint ($|\Delta y| \leq 0.5$\,m), and computing shortest paths via Dijkstra from the camera origin.

\subsection{Trajectory Planning Heuristics}
\noindent\textbf{Scoring Function for Target Paths.}
For target trajectories, 72 candidate nodes are sampled on a circular boundary at radius $r = \min(1.5 \times s_{obj}, d_{obj})$, where $s_{obj}$ is the object scale and $d_{obj}$ is the distance from the camera origin to the target. The optimal node is selected by maximizing a scoring function: $S = 5000 \cdot R_{los} - 100 \cdot \Delta\theta$, where $R_{los}$ is the Line-of-Sight ratio calculated via a 50-point projection test between the candidate position and the target center, and $\Delta\theta$ is the angular deviation from the ideal approach direction. Paths are computed via Dijkstra on the NavMesh graph, pruned of backtracking segments, and smoothed with B-spline interpolation (smoothing factor 0.5). All trajectories are resampled to 81 keyframes.

\noindent\textbf{Surround Behavior.}
For surround trajectories, the orbit radius is computed as $r = \min(d_{obj}, 2 \times s_{obj})$, clamped by a scene-dependent threshold ($4 \times$ global median depth). 72 points are sampled on the orbit circle; only those reachable via the NavMesh graph are connected to form continuous surround paths.

\noindent\textbf{Exploration Behavior.}
For wander trajectories, 8 radial directions are explored from the origin along the NavMesh, each extending up to $4 \times$ global median depth in path length. The top 10 trajectories are selected via Farthest Point Sampling on endpoint positions, with a minimum angular separation of 40$^{\circ}$ and maximum overlap ratio of 0.5.

\noindent\textbf{Reconstruction Behavior.}
Reconstruction targets are selected via the top-10 semantic objects ranked by segment area and viewpoint diversity (FPS in 3D space). For each target, multi-view paths are planned with a reduced smoothing factor (0.2) to preserve sharp viewpoint transitions optimal for 3D reconstruction coverage.

\noindent\textbf{Path Smoothing and Orientation.}
All discrete waypoints are smoothed using B-splines and Gaussian filtering ($\sigma=8.0$). Camera orientations are managed via Spherical Linear Interpolation (SLERP). For ``Wander'' and ``Target'' modes, orientations follow the path tangent and target centroid, respectively.

\section{Training and Inference Details}
\label{sec:impl_details}
The reported tokenizer uses $B=256$ bins and a normalized displacement limit $\Delta t_{\max}=0.025$. These settings determine precision and clipping tradeoffs; sign alignment does not remove quantization boundaries.

OmniCam is trained with the Accelerate framework under bf16 mixed precision. The autoregressive backbone is an OPT-style decoder equipped with FlashAttention-2. The optimizer is AdamW with $\beta = (0.9, 0.95)$, a weight decay of $0.01$, a peak learning rate of $3\times 10^{-5}$, a linear warm-up over the first one percent of the schedule, and cosine annealing for the remaining steps. The gradient is clipped to a global norm of $1.0$ and gradient checkpointing is enabled throughout the decoder. The discrete bin size is $B = 256$, and the maximum translation delta is $\Delta t_{\max} = 0.025$ in the normalized coordinate system. The hidden dimension, number of heads, and number of layers are configured as the base variant (hidden dim 1024, 8 heads, 12 layers). Training is performed for approximately one hundred effective epochs on a single node with eight NVIDIA H20 GPUs, with a per-GPU batch size of thirty-two. The localization loss weight is set to $\lambda = 0.5$. Inference adopts the \emph{sample} strategy with top-$k$ truncation at $k = 10$ and with the format-constrained decoding procedure described in \Cref{sec:generation}.

\noindent\textbf{Weighted Cross-Entropy Loss.} The per-token cross-entropy loss is weighted by three multiplicative factors to emphasize the most critical pose channels and learning phases: (1)~\emph{channel importance weights}: rotation channels receive weight 2.0, translation channels 1.5, and intrinsic channels 0.5, reflecting the relative importance of orientation accuracy over positional precision; (2)~\emph{temporal decay}: a linear schedule assigns weight 2.0 to the first frame and decreases to 1.0 at the final frame, prioritizing early trajectory establishment; and (3)~\emph{curriculum learning mask}: during the initial training phase, only early frames contribute to the loss, with later frames gradually enabled as training progresses, preventing the model from attempting to learn long-horizon dynamics before mastering short-horizon patterns.

\noindent\textbf{Caption Dropout Probabilities.} The five-variant caption dropout strategy applies the following distribution: full caption (35\%), prefix tags only (25\%), body text only (25\%), motion tag only (5\%), and partial prefix (10\%). This is the distribution used in the reported experiments.

\noindent\textbf{Computational Cost.} \Cref{tab:compute} reports training and inference costs for GenDoP and two OmniCam configurations. The original timing description includes depth estimation, point-cloud construction, encoding, and autoregressive decoding. The prefix condition is reused across decoding steps. These reported timings do not establish real-time operation under an unspecified deployment budget.

\begin{table*}[tp]
\centering
\small
\caption{Reported computational costs for GenDoP and OmniCam Base/Large.}
\label{tab:compute}
\begin{tabular}{l|ccc}
\toprule
\textbf{Method} & \textbf{Training Time} & \textbf{Inference Latency} & \textbf{\#Params} \\
\midrule
GenDoP & $\sim$18h (8$\times$H20) & 0.42s / trajectory & $\sim$380M \\
OmniCam (Base) & $\sim$24h (8$\times$H20) & 0.84s / trajectory & $\sim$500M \\
OmniCam (Large) & $\sim$40h (8$\times$H20) & 1.20s / trajectory & $\sim$1B \\
\bottomrule
\end{tabular}
\end{table*}

\section{Encoder Architecture Details}
\label{sec:encoder_details}
The dual-branch encoder routes appearance and text through a semantic pathway and point-cloud features through a geometry pathway. The outputs are fused by the decoder.

\noindent\textbf{Semantic Branch.} The panoramic image $\mathbf{I}_{pano}$ is processed by a frozen SAM\,3 vision encoder with 1024-dimensional features. The target descriptor $T_{obj}$ and caption guide the Language-Guided Resampler after visual feature extraction. This pathway provides target-related information, but it is not excluded from contributing to geometric decisions after decoder fusion.

\noindent\textbf{Geometric Branch.} Panoramic geometric priors are unprojected into a 3D point cloud, which is then structured into a voxel grid ($0.02$m resolution) to enable efficient feature extraction. A LitePT encoder uses positional and surface-normal features to extract per-voxel descriptors via sparse convolution, which are re-projected to yield $\mathbf{F}_{geo} \in \mathbb{R}^{M \times C_{geo}}$. This branch has no direct text input. Removing it increases collision rate from 10.4\% to 45.8\% in \Cref{tab:merged_ablation}; this supports the usefulness of the geometric pathway but does not isolate routing from capacity or input-information changes.

\section{Per-Behavior Decomposition}
\label{sec:per_behavior}
\Cref{tab:per_behavior} decomposes the OmniCaT results. Compared with retrained GenDoP, OmniCam has lower ATE for Wander (1.150 vs.\ 2.421), Target (0.677 vs.\ 1.088), Surround (1.258 vs.\ 1.932), and Reconstruct (0.385 vs.\ 1.055). The corresponding Target and Reconstruct reductions are 37.8\% and 63.5\%. This comparison supports performance across the four evaluated behaviors; it does not isolate caption dropout, loss weighting, or geometry as the cause of individual gains.

\begin{table*}[tp]
\centering
\small
\caption{Per-behavior decomposition of the main result on OmniCaT. Metrics are reported separately for Wander, Target, Surround, and Reconstruct trajectories.}
\label{tab:per_behavior}
\resizebox{\linewidth}{!}{
\begin{tabular}{l|c|cc|cc|cc|cc}
\toprule
\multirow{2}{*}{\textbf{Method}} & \multirow{2}{*}{\textbf{Dataset}} & \multicolumn{2}{c|}{\textbf{Wander}} & \multicolumn{2}{c|}{\textbf{Target}} & \multicolumn{2}{c|}{\textbf{Surround}} & \multicolumn{2}{c}{\textbf{Reconstruct}} \\
 & & \textbf{ATE $\downarrow$} & \textbf{FDE~$\downarrow$} & \textbf{ATE $\downarrow$} & \textbf{FDE~$\downarrow$} & \textbf{ATE $\downarrow$} & \textbf{FDE~$\downarrow$} & \textbf{ATE $\downarrow$} & \textbf{FDE~$\downarrow$} \\
\midrule
CCD~\cite{jiang2024cinematographic} & Pre-trained & 3.718 & 4.907 & 1.485 & 2.285 & 2.524 & 3.419 & 2.206 & 2.972 \\
E.T.~\cite{courant2024exceptional}  & Pre-trained & 3.151 & 4.083 & 1.735 & 1.945 & 1.580 & 1.640 & 1.467 & 1.804 \\
Director3D~\cite{li2024director3d}  & Pre-trained & 2.969 & 4.779 & 1.444 & 2.207 & 1.620 & 2.688 & 1.279 & 2.094 \\
Director3D  & OmniCaT & 2.562 & 4.250 & 1.130 & 2.089 & 1.351 & 2.304 & 1.045 & 1.950 \\
GenDoP~\cite{zhang2025gendop}       & Pre-trained & 2.667 & 4.408 & 1.362 & 2.131 & 2.176 & 3.248 & 1.255 & 2.016 \\
GenDoP & OmniCaT & 2.421 & 3.652 & 1.088 & 1.933 & 1.932 & 2.773 & 1.055 & 1.574 \\
\midrule
\cellcolor[HTML]{efefef}\textbf{OmniCam (Ours)} & \cellcolor[HTML]{efefef}\textbf{OmniCaT} & \cellcolor[HTML]{efefef}\textbf{1.150} & \cellcolor[HTML]{efefef}\textbf{2.301} & \cellcolor[HTML]{efefef}\textbf{0.677} & \cellcolor[HTML]{efefef}\textbf{0.962} & \cellcolor[HTML]{efefef}\textbf{1.258} & \cellcolor[HTML]{efefef}\textbf{1.836} & \cellcolor[HTML]{efefef}\textbf{0.385} & \cellcolor[HTML]{efefef}\textbf{0.749} \\
\bottomrule
\end{tabular}
}
\end{table*}

\section{Scaling Analysis}
\label{sec:scaling}
We examine the reported dependence on training-data volume and model size. These finite comparisons show trends within the evaluated range; they do not establish a scaling law or the absence of representational bottlenecks.

\subsection{Data Scaling}
\label{sec:data_scaling}
\Cref{tab:data_scaling} reports the base model at four nominal data scales, labeled 27K, 53K, 107K, and 267K in the original experiment record. Metrics improve across these settings: ATE decreases from 1.483 to 0.868, collision from 19.8\% to 10.4\%, and visibility increases from 34.6\% to 52.0\%. Four data points do not determine a power law or establish gains beyond the observed range.

\begin{table*}[tp]
\centering
\small
\caption{Reported effect of nominal data scale using the base architecture. Exact training-subset counts and the fixed evaluation split require verification.}
\label{tab:data_scaling}
\begin{tabular}{r|cc|cc|c|c}
\toprule
\textbf{\#Trajectories} & \textbf{ATE $\downarrow$} & \textbf{FDE $\downarrow$} & \textbf{RPE-R $\downarrow$} & \textbf{RPE-T $\downarrow$} & \textbf{Coll.~(\%) $\downarrow$} & \textbf{Vis.~(\%) $\uparrow$} \\
\midrule
27K (10\%)   & 1.483 & 2.586 & 1.836 & 0.054 & 19.8 & 34.6 \\
53K (20\%)   & 1.206 & 2.080 & 1.528 & 0.042 & 15.8 & 42.6 \\
107K (40\%)  & 1.003 & 1.709 & 1.297 & 0.033 & 12.7 & 48.2 \\
\cellcolor[HTML]{efefef}\textbf{267K (100\%)} & \cellcolor[HTML]{efefef}\textbf{0.868} & \cellcolor[HTML]{efefef}\textbf{1.462} & \cellcolor[HTML]{efefef}\textbf{1.135} & \cellcolor[HTML]{efefef}\textbf{0.027} & \cellcolor[HTML]{efefef}\textbf{10.4} & \cellcolor[HTML]{efefef}\textbf{52.0} \\
\bottomrule
\end{tabular}
\end{table*}

\subsection{Model Scaling}
\label{sec:model_scaling}
Three model sizes are compared: Small (hidden dimension 512, 8 heads, 8 layers; approximately 120M trainable parameters), Base (1024, 8, 12; approximately 500M), and Large (1536, 16, 24; approximately 1B). The reported setup holds other hyperparameters constant while changing architecture dimensions. ATE decreases from 1.127 to 0.743, collision from 13.8\% to 8.7\%, and visibility increases from 44.6\% to 56.3\% (\Cref{tab:model_scaling}). These observations support gains at the tested sizes, without establishing which representational choices explain the trend.

\begin{table*}[tp]
\centering
\small
\caption{Reported effect of model scale. The common training subset and parameter-count scope require verification against training records.}
\label{tab:model_scaling}
\begin{tabular}{l|r|cc|cc|c|c}
\toprule
\textbf{Size} & \textbf{\#Params} & \textbf{ATE $\downarrow$} & \textbf{FDE $\downarrow$} & \textbf{RPE-R $\downarrow$} & \textbf{RPE-T $\downarrow$} & \textbf{Coll.~(\%) $\downarrow$} & \textbf{Vis.~(\%) $\uparrow$} \\
\midrule
Small & $\sim$120M & 1.127 & 1.893 & 1.428 & 0.036 & 13.8 & 44.6 \\
Base  & $\sim$500M & 0.868 & 1.462 & 1.135 & 0.027 & 10.4 & 52.0 \\
\cellcolor[HTML]{efefef}\textbf{Large} & \cellcolor[HTML]{efefef}$\sim$\textbf{1B}   & \cellcolor[HTML]{efefef}\textbf{0.743} & \cellcolor[HTML]{efefef}\textbf{1.248} & \cellcolor[HTML]{efefef}\textbf{0.982} & \cellcolor[HTML]{efefef}\textbf{0.022} & \cellcolor[HTML]{efefef}\textbf{8.7} & \cellcolor[HTML]{efefef}\textbf{56.3} \\
\bottomrule
\end{tabular}
\end{table*}

\section{Failure Case Analysis}
\label{sec:failure_cases}
Understanding when and why the framework fails is essential for guiding future improvements. Five categories of failures are identified, each linked to a specific component or assumption of the pipeline:
\begin{itemize}
\item \textbf{Noisy panoramic geometry / thin structures} (geometric branch limitation): in large open environments or near thin structures (fences, poles, railings), the monocular depth estimator~\cite{wang2025moge} produces noisy or incomplete reconstructions with floating artifacts. These propagate into the LitePT voxel features, causing the geometric branch to either miss collision boundaries or hallucinate obstacles, resulting in physically implausible paths.
\item \textbf{Ambiguous target phrases} (centroid injection limitation): when multiple semantically similar objects coexist (e.g., ``the wooden door'' in a hallway with several doors, or ``the pillar'' among clustered columns), the attention-aggregated centroid may ground to an incorrect instance. Spatial descriptors in the textual guidance partially mitigate this, but the failure persists when such cues are absent or when objects are visually indistinguishable from the panoramic viewpoint.
\item \textbf{Dynamic and interactive scenes} (static-scene assumption): the current framework assumes a static panoramic observation. In environments with moving objects (people, vehicles, opening doors), the generated trajectory cannot react to temporal changes, potentially producing paths that collide with dynamic obstacles or fail to track moving targets. Extending OmniCam to video-conditioned or closed-loop settings is an important direction for future work.
\item \textbf{Reflective/transparent surfaces} (semantic branch limitation): glass walls and mirrors produce ambiguous visual features in the vision encoder, leading to unreliable target-aware conditioning and erratic trajectory orientation near such surfaces.
\item \textbf{Occluded targets} (centroid injection limitation): when the target object is heavily occluded, the cross-attention weights spread to visible but irrelevant regions, causing the estimated centroid to shift away from the true target location.
\end{itemize}

These cases identify limitations in perception and target grounding. Decoder errors, quantization, and exposure to planner-generated supervision can also contribute; the examples do not isolate their relative importance.

\noindent\textbf{Broader Impacts.} Scene-aware camera planning may assist content creation and robotic perception. Camera deployment can also raise privacy or surveillance concerns, and generated video may be used misleadingly. In robotics, errors in inferred geometry or open-loop trajectories can produce unsafe camera motions; the reported collision metrics are not deployment guarantees. Dataset release requires source-specific permission and license review, and physical deployment requires validation against the actual environment.

\section{Attention Visualization and Centroid Analysis}
\label{sec:attn_vis}
The target centroid supplies an explicit spatial anchor (\Cref{sec:resampling}). The following visualizations illustrate the attention patterns used to derive this anchor.

\Cref{fig:attn_vis} shows the input panorama, reconstructed point cloud, and attention projected onto scene points. The examples show concentration near referenced regions. These visualizations complement the grounding ablation, but do not quantify localization accuracy or establish that attention alone causes improved visibility. Point-level aggregation is specified conceptually in \Cref{eq:target_centroid}; its implementation details require verification.

\begin{figure*}[tp]
\centering
\includegraphics[width=\linewidth]{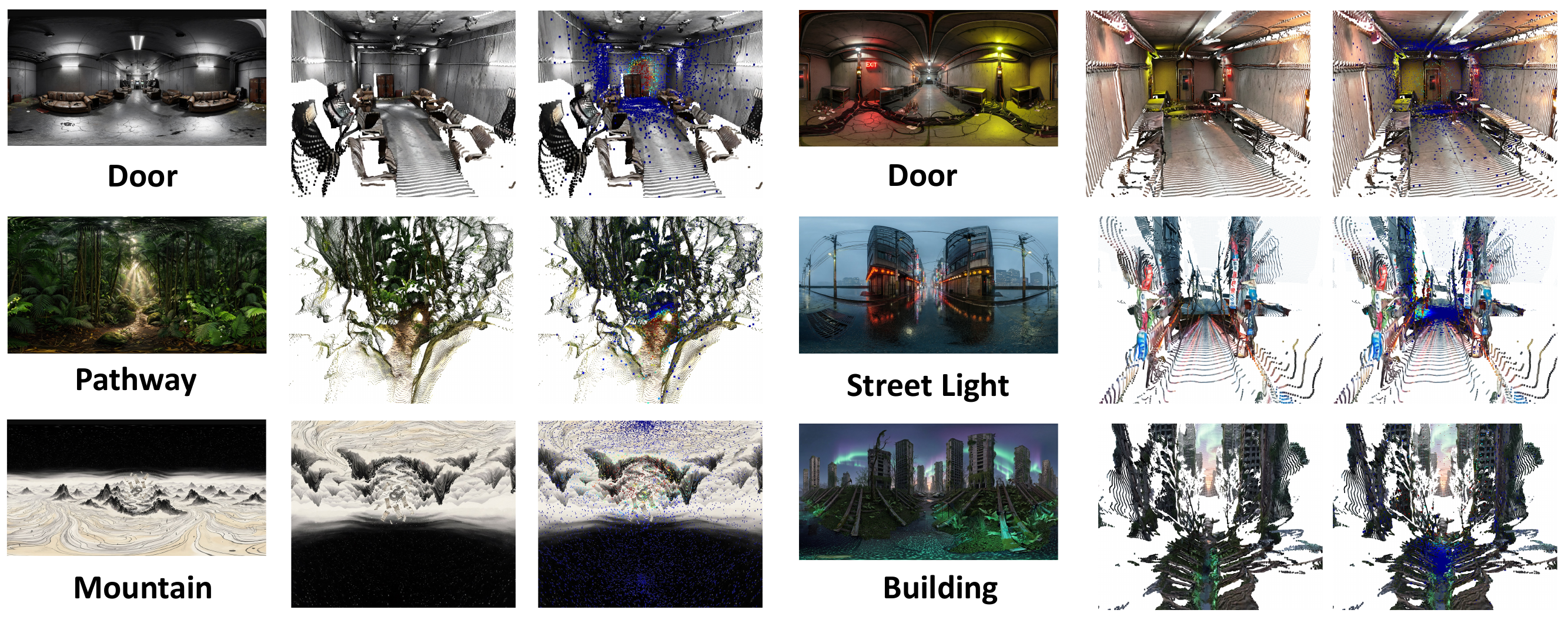}
\caption{\textbf{Attention visualization of the Language-Guided Semantic Resampler.} Each group shows (from left to right) the input panorama, the reconstructed point cloud, and the 3D attention map. The maps visualize target-related attention on reconstructed geometry for the displayed examples; they are qualitative illustrations rather than a quantitative localization evaluation.}
\label{fig:attn_vis}
\end{figure*}

\section{Caption Dropout Ablation}
\label{sec:caption_dropout}
We compare no caption dropout, single-variant dropout (dropping the entire caption with probability 0.1), and the five-variant distribution in \Cref{sec:impl_details}. The reported values are means over three training seeds, with ATE standard deviation below 0.02.
The five-variant scheme has higher visibility than no dropout (52.0\% vs.\ 38.7\%) and lower angular deviation (10.9$^{\circ}$ vs.\ 18.4$^{\circ}$), while no dropout has slightly lower ATE (0.856 vs.\ 0.868) and collision (10.1\% vs.\ 10.4\%). This is a tradeoff between the reported metrics. Attribution to instruction-format overfitting requires separate evaluation on complete, partial, and rephrased instructions.

\begin{table*}[tp]
\centering
\small
\caption{Caption-dropout ablation. The five-variant scheme favors visibility and angle, while no dropout has slightly lower geometric errors.}
\label{tab:caption_dropout}
\setlength{\tabcolsep}{4.5pt}
\begin{tabular}{l|cc|cc|c|cc}
\toprule
\textbf{Dropout Strategy} & \textbf{ATE $\downarrow$} & \textbf{FDE $\downarrow$} & \textbf{RPE-R $\downarrow$} & \textbf{RPE-T $\downarrow$} & \textbf{Coll.~(\%) $\downarrow$} & \textbf{Vis.~(\%) $\uparrow$} & \textbf{Angle~($^{\circ}$) $\downarrow$} \\
\midrule
No dropout          & 0.856 & 1.438 & 1.118 & 0.026 & 10.1 & 38.7 & 18.4 \\
Single-variant      & 0.862 & 1.451 & 1.126 & 0.027 & 10.3 & 45.8 & 14.2 \\
\cellcolor[HTML]{efefef}\textbf{Full (5-variant)}    & \cellcolor[HTML]{efefef}\textbf{0.868} & \cellcolor[HTML]{efefef}\textbf{1.462} & \cellcolor[HTML]{efefef}\textbf{1.135} & \cellcolor[HTML]{efefef}\textbf{0.027} & \cellcolor[HTML]{efefef}\textbf{10.4} & \cellcolor[HTML]{efefef}\textbf{52.0} & \cellcolor[HTML]{efefef}\textbf{10.9} \\
\bottomrule
\end{tabular}
\end{table*}

\section{Theoretical Analysis of Hybrid Pose Tokenization}
\label{sec:theory_appendix}

We analyze reconstruction of a known trajectory after quantization and dequantization. These results concern representation error; they do not bound autoregressive prediction error, depth-estimation error, or errors in scene-scale calibration. Let $d$ denote the rotation-angle geodesic distance on $\mathrm{SO}(3)$, measured in radians, with range $[0,\pi]$. This distance is invariant under left and right multiplication by a rotation.

\begin{proposition}[Deterministic bounds on rotation reconstruction error]
\label{prop:rotation_error}
Let $R_1,\ldots,R_T\in\mathrm{SO}(3)$ and $\Delta R_t=R_{t-1}^{-1}R_t$ for $t\ge2$. Suppose decoded relative rotations satisfy $d(\widehat{\Delta R}_t,\Delta R_t)\le\epsilon_t$, and the decoded initial rotation satisfies $d(\widehat R_1,R_1)\le e_1$. Reconstruct relative orientations by $\widehat R_t=\widehat R_{t-1}\widehat{\Delta R}_t$. Then
\begin{equation}
  d(\widehat R_t,R_t)\le\min\!\left\{\pi,\ e_1+\sum_{j=2}^{t}\epsilon_j\right\}.
\end{equation}
If $\widehat R_1=R_1$ and every $\epsilon_j\le\epsilon$, this bound becomes $\min\{\pi,(t-1)\epsilon\}$. In contrast, if absolute orientations are decoded with per-frame reconstruction error at most $\epsilon$, then $d(\widehat R_t^{\mathrm{abs}},R_t)\le\epsilon$ at every frame, without integrating earlier quantization residuals.
\end{proposition}

\begin{proof}
The triangle inequality and bi-invariance give
\begin{align}
 d(\widehat R_t,R_t)
 &\le d(\widehat R_{t-1}\widehat{\Delta R}_t,\widehat R_{t-1}\Delta R_t)\notag\\
 &\quad+d(\widehat R_{t-1}\Delta R_t,R_{t-1}\Delta R_t)\\
 &=d(\widehat{\Delta R}_t,\Delta R_t)+d(\widehat R_{t-1},R_{t-1}).
\end{align}
Iterating this inequality proves the sum bound. Every rotation-angle distance is at most $\pi$, giving the stated minimum. The absolute-encoding bound follows directly from its assumed per-frame reconstruction guarantee.
\end{proof}

This is a worst-case upper bound, not a claim that relative errors necessarily grow at each frame or obey a particular expected growth rate. Residuals can cancel. Absolute encoding avoids integration of rotation quantization residuals, but predicted absolute tokens can still depend on previous generated tokens and therefore need not have independent prediction errors.

\noindent\textbf{Connecting component quantization to angular error.}
The bound $\epsilon$ depends on the quantizer and the rotation decoder; $B$ component bins do not by themselves imply $\epsilon=\pi/B$. For a conditional example, let $q\in\mathbb S^3$ be a unit quaternion, and suppose component dequantization produces $v$ with $\|v-q\|_2\le\eta<1$. If the decoder normalizes $v$ to $\widehat q=v/\|v\|_2$, then
\begin{equation}
 d(R(\widehat q),R(q))\le2\arcsin\eta.
\end{equation}
Indeed, $q^\top v>0$, so the angle $\phi$ between $q$ and $v$ lies in $[0,\pi/2)$. Decomposing the error along $\widehat q$ gives $\|v-q\|_2^2=(\|v\|_2-\cos\phi)^2+\sin^2\phi$, hence $\phi\le\arcsin\eta$. The corresponding rotation angle is $2\phi$. For uniform bins of width $h$, center decoding without clipping gives at most $h/2$ error per quaternion component and thus $\eta\le h$. In particular, $B>2$ equal bins on $[-1,1]$ give the conditional bound $2\arcsin(2/B)$. Applying such a numerical bound requires these decoder and range assumptions to hold.

\begin{proposition}[Representative-independent quaternion sign selection]
\label{prop:hemisphere}
Let $R_1,\ldots,R_T$ satisfy $d(R_{t-1},R_t)<\pi$ for every $t\ge2$, and choose arbitrary unit-quaternion representatives $q_t$ of $R_t$. For $q=(w,x,y,z)$, let $c(q)$ select the sign for which the first nonzero component in the order $(w,x,y,z)$ is positive. Define
\begin{equation}
 \begin{gathered}
 \widetilde q_1=c(q_1),\\
 \widetilde q_t=
 \begin{cases}
 q_t,&\langle\widetilde q_{t-1},q_t\rangle>0,\\
 -q_t,&\langle\widetilde q_{t-1},q_t\rangle<0,
 \end{cases}\quad t\ge2.\end{gathered}
\end{equation}
This produces a unique aligned quaternion sequence determined by the rotation sequence, independently of the initial choices between $q_t$ and $-q_t$. Consequently, any deterministic component quantizer with fixed bin-boundary rules produces the same token sequence from all such choices of input representatives. The tokenization remains many-to-one.
\end{proposition}

\begin{table*}[tp]
\centering
\small
\caption{\textbf{Privileged-information planner comparison.} The planner receives reconstructed geometry, navigation meshes, and semantic annotations. OmniCam estimates geometry from the input panorama and text.}
\label{tab:planner_oracle}
\resizebox{\linewidth}{!}{
\begin{tabular}{l|c|c|ccc}
\toprule
\textbf{Method} & \textbf{NavMesh/annotation access} & \textbf{Runtime} & \textbf{Coll.~(\%) $\downarrow$} & \textbf{Vis.~(\%) $\uparrow$} & \textbf{Angle~($^{\circ}$) $\downarrow$} \\
\midrule
Heuristic Planner & Yes (privileged) & $\sim$16.5s & 5.2 & 58.7 & 9.4 \\
GenDoP~\cite{zhang2025gendop} & No & 0.42s & 30.4 & 33.6 & 27.6 \\
{OmniCam (Ours)} & No & 0.84s &{10.4} & {52.0} &{10.9} \\
\bottomrule
\end{tabular}
}
\end{table*}

\begin{proof}
A unit quaternion has a nonzero component, and $c(q)=c(-q)$, so the first representative is uniquely fixed, including when $w_1=0$. Suppose the aligned representative at $t-1$ is fixed. The quaternion formula for rotation distance gives
\[
 |\langle\widetilde q_{t-1},q_t\rangle|
 =\cos\!\left(\frac{d(R_{t-1},R_t)}{2}\right)>0.
\]
Exactly one of $q_t$ and $-q_t$ therefore has positive inner product with the already aligned previous quaternion. Replacing the input $q_t$ by $-q_t$ reverses the inner-product sign and leaves the selected representative unchanged. Induction proves representative independence of the entire sequence, and deterministic quantization preserves this property. It does not preserve injectivity: a finite token vocabulary cannot distinguish every continuous rotation sequence.
\end{proof}

The proposition specifies sufficient sign and boundary conventions; a concrete implementation must document these choices. At an exactly $\pi$ inter-frame rotation, both inner products are zero and an additional deterministic tie rule is required, such as selecting $c(q_t)$. No shortest-arc preference exists at that tie. This sequence-level sign selection removes double-cover ambiguity under the stated conventions; it does not construct a globally continuous quaternion parameterization of $\mathrm{SO}(3)$, remove quantization boundaries, or make lossy quantization invertible.

\noindent\textbf{Translation: local resolution versus accumulated error.}
For additive displacements in a common coordinate frame, relative translation trades a smaller encoding range for integration error. If normalized absolute coordinates lie in $[-L,L]$ and displacements lie in $[-D,D]$ per component, uniform center decoding with $B$ bins and no clipping bounds the per-frame absolute position error by $\sqrt3L/B$ and each displacement error by $\sqrt3D/B$. Integrating the latter gives
\begin{equation}
 \|\widehat{\mathbf p}_t-\mathbf p_t\|_2
 \le\|\widehat{\mathbf p}_1-\mathbf p_1\|_2+(t-1)\frac{\sqrt3D}{B}.
\end{equation}
Thus $D\ll L$ improves local displacement resolution, while long sequences can accumulate translation error. If displacements are expressed in rotating local frames, orientation errors must also be included. The hybrid representation is therefore a design tradeoff whose trajectory-level benefit requires empirical evaluation, rather than a guarantee of universally minimal pose error.

\section{Heuristic Planner Oracle Comparison}
\label{sec:planner_oracle}

The heuristic planner is evaluated with reconstructed point clouds, navigation meshes, semantic masks, and target annotations. This is a privileged-information comparison on estimated geometry, rather than an oracle with complete physical ground truth. It attains 5.2\% collision and 58.7\% visibility, versus OmniCam's 10.4\% and 52.0\%. OmniCam estimates geometry internally from the panorama and does not run navmesh construction or graph search. The reported runtimes (approximately 16.5s vs.\ 0.84s) differ by a factor of 19.6; the timing assumptions must be matched before attributing the difference to a particular component.

\section{More Qualitative Results}
\label{sec:more_qualitative}

\Cref{fig:more_qualitative_comparison} provides additional qualitative trajectory comparisons. Each row shows one scene, and the columns show the outputs of different methods. These selected examples illustrate differences in trajectory placement and orientation; they do not establish dataset-wide collision or smoothness guarantees.

\begin{figure*}[tp]
\centering
\includegraphics[width=\linewidth,height=0.82\textheight,keepaspectratio]{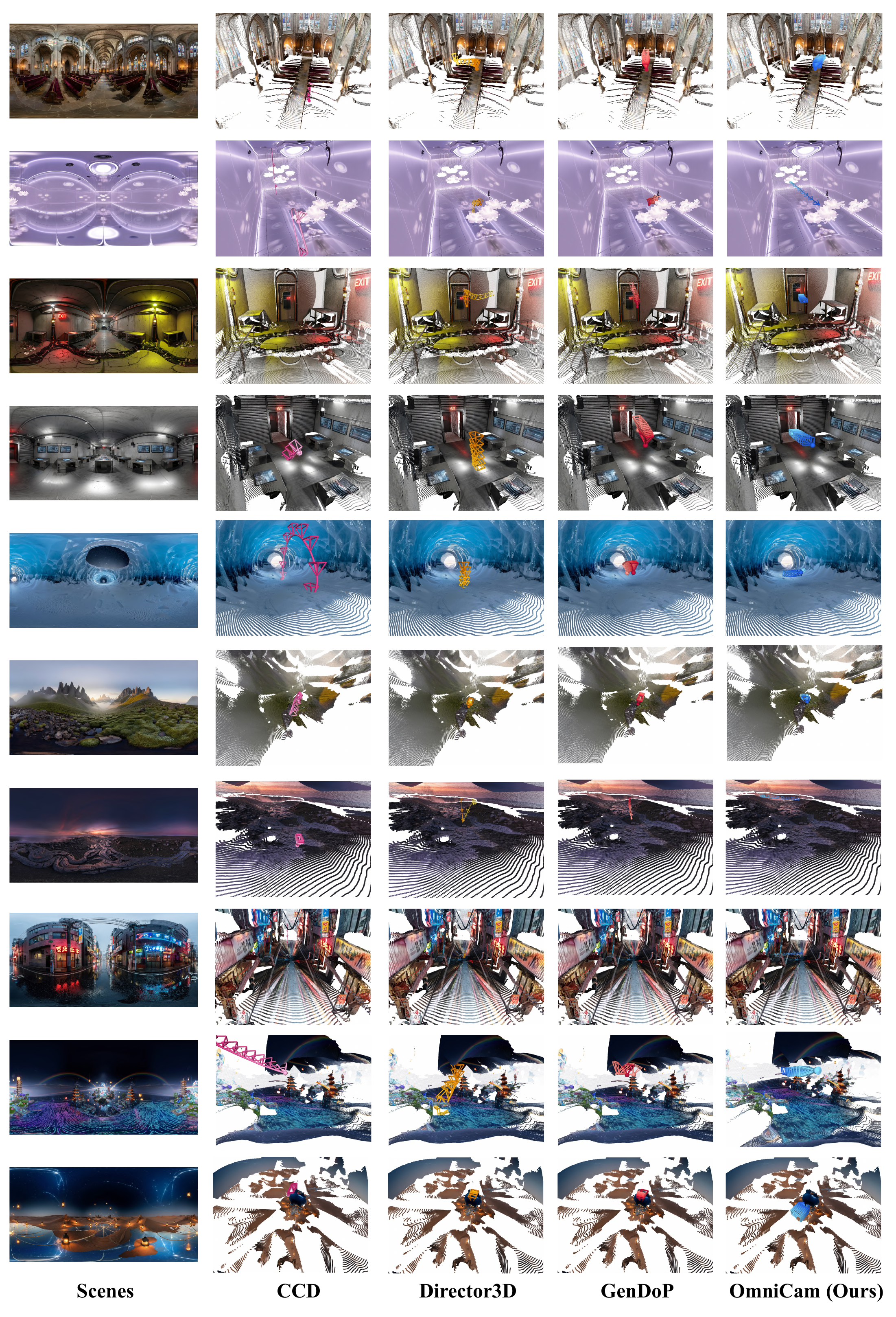}
\caption{\textbf{Additional qualitative comparison on panoramic scenes.} Extending the comparison in \Cref{fig:qualitative}, each row shows a different scene spanning indoor, outdoor, and stylized environments, and each column shows the trajectory produced by a specific method. The examples illustrate differences in target-directed trajectories across scene categories. The trajectory captions used for each scene are listed in \Cref{tab:qualitative_captions}.}
\label{fig:more_qualitative_comparison}
\end{figure*}

\begin{table*}[tp]
\centering
\small
\caption{\textbf{Trajectory captions for \Cref{fig:more_qualitative_comparison}.} Each row corresponds to a scene in the qualitative comparison figure, listing the scene type, target object, and the hierarchical textual instruction provided to OmniCam.}
\label{tab:qualitative_captions}
\setlength{\tabcolsep}{4pt}
\begin{tabularx}{\textwidth}{@{}c|l|l|>{\raggedright\arraybackslash}X@{}}
\toprule
\textbf{Row} & \textbf{Type} & \textbf{Target} & \textbf{Caption} \\
\midrule
1 & Target & door & \texttt{<TASK: TARGET> <TARGET\_OBJ: door> <TARGET\_DIR: Back>} Approach the back door to bring it into focus. The sequence concludes by arriving at the final \texttt{<END\_OBJ: door>} position within the \texttt{<END\_DIR: Back>} sector of the panoramic scene. \\
\midrule
2 & Target & light fixture & \texttt{<TASK: TARGET> <TARGET\_OBJ: light fixture> <TARGET\_DIR: Front>} Approach the front light fixture to bring it into focus. The sequence concludes by arriving at the final \texttt{<END\_OBJ: light fixture>} position within the \texttt{<END\_DIR: Front>} sector of the panoramic scene. \\
\midrule
3 & Wander & free space & \texttt{<TASK: WANDER> <TARGET\_OBJ: free space> <TARGET\_DIR: Front>} Approach the front workstation to explore the layout of the industrial corridor. The sequence concludes by arriving at the final \texttt{<END\_OBJ: free space>} position within the \texttt{<END\_DIR: Front>} sector of the panoramic scene. \\
\midrule
4 & Target & door & \texttt{<TASK: TARGET> <TARGET\_OBJ: door> <TARGET\_DIR: Back>} Approach the back door to bring it into focus. The sequence concludes by arriving at the final \texttt{<END\_OBJ: door>} position within the \texttt{<END\_DIR: Back>} sector of the panoramic scene. \\
\midrule
5 & Target & ice tunnel & \texttt{<TASK: TARGET> <TARGET\_OBJ: ice tunnel> <TARGET\_DIR: Front>} Approach the ice tunnel to bring it into focus. The sequence concludes by arriving at the final \texttt{<END\_OBJ: ice tunnel>} position within the \texttt{<END\_DIR: Front>} sector of the panoramic scene. \\
\midrule
6 & Wander & free space & \texttt{<TASK: WANDER> <TARGET\_OBJ: free space> <TARGET\_DIR: Front>} Approach the distant peaks to explore the mountain valley. The sequence concludes by arriving at the final \texttt{<END\_OBJ: free space>} position within the \texttt{<END\_DIR: Front>} sector of the panoramic scene. \\
\midrule
7 & Surround & clouds & \texttt{<TASK: SURROUND> <TARGET\_OBJ: clouds> <TARGET\_DIR: Front>} Capture the clouds from multiple angles to highlight their form and color against the twilight sky. The sequence concludes by arriving at the final \texttt{<END\_OBJ: clouds>} position within the \texttt{<END\_DIR: Front>} sector of the panoramic scene. \\
\midrule
8 & Surround & neon sign & \texttt{<TASK: SURROUND> <TARGET\_OBJ: neon sign> <TARGET\_DIR: Right>} Capture the right neon sign from multiple angles to highlight its glowing details and reflections. The sequence concludes by arriving at the final \texttt{<END\_OBJ: neon sign>} position within the \texttt{<END\_DIR: Right>} sector of the panoramic scene. \\
\midrule
9 & Wander & free space & \texttt{<TASK: WANDER> <TARGET\_OBJ: free space> <TARGET\_DIR: Right>} Approach the right to explore the surreal landscape beyond the pagoda. The sequence concludes by arriving at the final \texttt{<END\_OBJ: free space>} position within the \texttt{<END\_DIR: Right>} sector of the panoramic scene. \\
\midrule
10 & Surround & sand dune & \texttt{<TASK: SURROUND> <TARGET\_OBJ: sand dune> <TARGET\_DIR: Front>} Capture the sand dune from multiple angles to reveal its textured ridges and surrounding lanterns. The sequence concludes by arriving at the final \texttt{<END\_OBJ: sand dune>} position within the \texttt{<END\_DIR: Front>} sector of the panoramic scene. \\
\bottomrule
\end{tabularx}
\end{table*}

\begin{table*}[tp]
\centering
\small
\caption{\textbf{Downstream evaluation: camera-controlled video generation with WorldStereo~\cite{zhang2026worldstereo}.} Different trajectory sources are compared using the same video generation model.}
\label{tab:downstream_video}
\resizebox{\linewidth}{!}{
\begin{tabular}{l|ccccc}
\toprule
\textbf{Method} & \textbf{CLIP-S $\uparrow$} & \textbf{Adj.-SSIM $\uparrow$} & \textbf{Aest. proxy $\uparrow$} & \textbf{Fidelity $\uparrow$} & \textbf{Coverage} \\
\midrule
GenDoP~\cite{zhang2025gendop} + WorldStereo & 0.2644 & 0.4248 & 7.647 & 0.6924 & 35.8\% \\
\cellcolor[HTML]{efefef}\textbf{OmniCam (Ours) + WorldStereo} & \cellcolor[HTML]{efefef}\textbf{0.3018} & \cellcolor[HTML]{efefef}\textbf{0.4904} & \cellcolor[HTML]{efefef}\textbf{9.021} & \cellcolor[HTML]{efefef}\textbf{0.7560} & \cellcolor[HTML]{efefef}38.5\% \\
\bottomrule
\end{tabular}
}
\end{table*}

\begin{table*}[tp]
\centering
\small
\caption{\textbf{Downstream evaluation: robotic manipulation with $\pi_{0.5}$.} Reported grasp success (SR) and execution visibility (Vis.). The stated 200-episode setup and fixed-camera SR aggregation require reconciliation with the episode records.}
\label{tab:downstream_robot}
\begin{tabular}{l|cc}
\toprule
\textbf{Trajectory Source} & \textbf{Grasp SR~(\%) $\uparrow$} & \textbf{Exec.\ Vis.~(\%) $\uparrow$} \\
\midrule
$\pi_{0.5}$~\cite{intelligence2025pi} (fixed camera) & 56.3 & 47.2 \\
GenDoP~\cite{zhang2025gendop} & 48.0 & 52.6 \\
\cellcolor[HTML]{efefef}\textbf{OmniCam (Ours)} & \cellcolor[HTML]{efefef}\textbf{72.5} & \cellcolor[HTML]{efefef}\textbf{71.8} \\
\bottomrule
\end{tabular}
\end{table*}

\Cref{fig:more_qualitative} presents additional downstream examples from indoor, outdoor, and stylized environments. Each row contains a panorama, selected generated video frames, and reconstruction results with and without the OmniCam trajectory. The figure illustrates qualitative differences; reconstruction accuracy and completeness require a separate quantitative evaluation.

\begin{figure*}[tp]
\centering
\includegraphics[width=\linewidth,height=0.83\textheight,keepaspectratio]{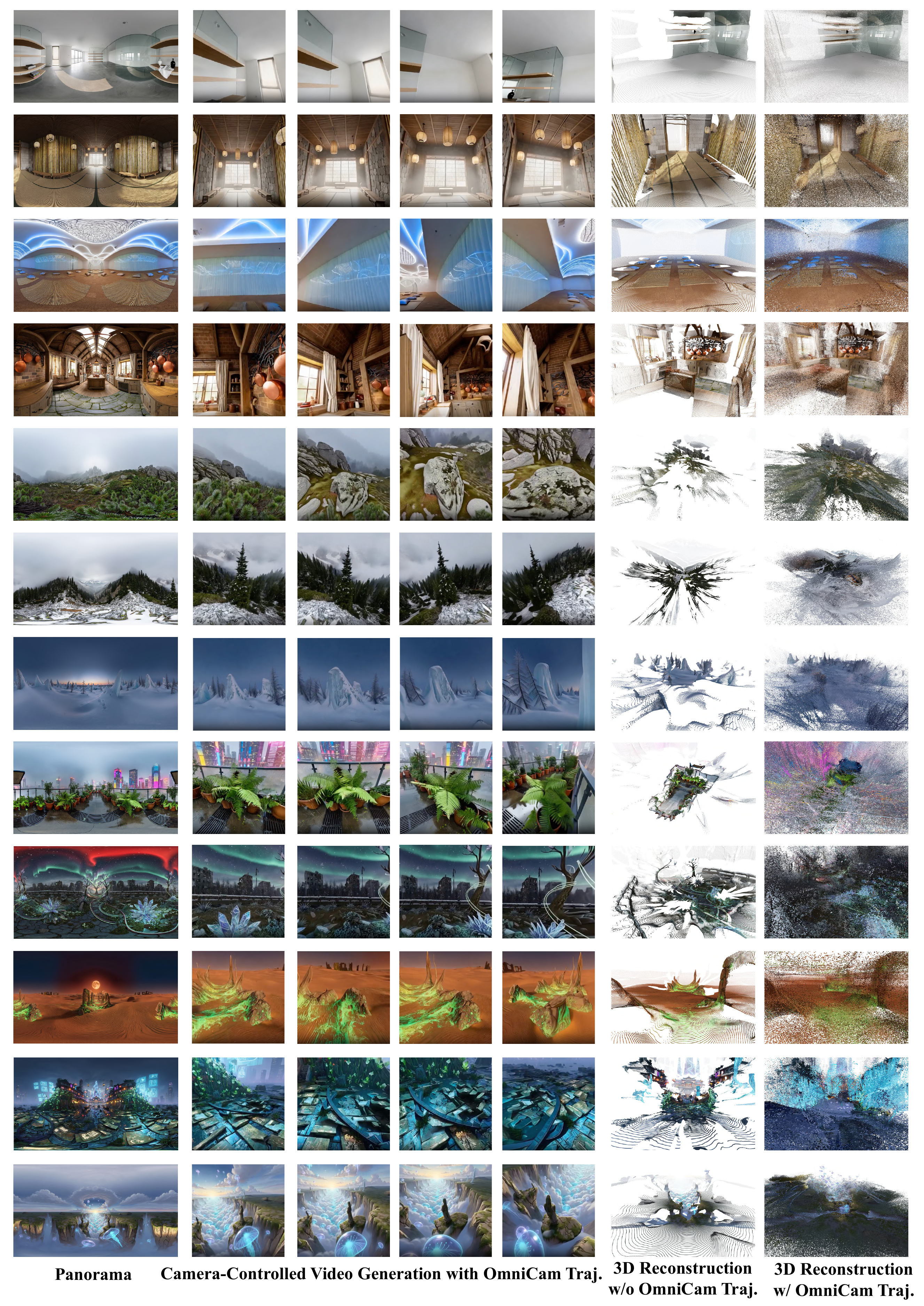}
\caption{\textbf{Qualitative downstream results across diverse panoramic scenes.} Each row shows the input panorama (left), four selected frames from camera-controlled video generation using an OmniCam trajectory (middle), and 3D reconstruction results without and with the OmniCam trajectory (two rightmost columns). The examples include indoor, outdoor, and stylized environments. These visual comparisons illustrate the pipeline outputs; they do not quantify reconstruction accuracy or guarantee collision-free motion.}
\label{fig:more_qualitative}
\end{figure*}

\section{Downstream Application Validation}
\label{sec:downstream_appendix}

\noindent\textbf{Camera-Controlled Video Generation.}
OmniCam and GenDoP trajectories serve as camera control signals for WorldStereo~\cite{zhang2026worldstereo}. The reported evaluation uses 484 trajectories across indoor, outdoor, and stylized scenes. For each prompt, the resulting poses drive point-cloud rendering and subsequent video synthesis.

Generated videos are evaluated with five reported measures. \emph{CLIP-Score} averages the cosine similarity between CLIP ViT-B/32 image and scene-description embeddings over 10 sampled frames. \emph{Consecutive-frame SSIM} measures adjacent-frame similarity and is sensitive to motion magnitude; it is not a motion-compensated temporal-consistency metric. \emph{Aesthetic proxy} combines CLIP embedding magnitude and inter-frame embedding coherence on a reported 0--10 scale. Its formula and normalization require verification. \emph{Known-region Fidelity} compares generated frames with point-cloud renderings in observed regions using SSIM, so it uses a rendering reference. \emph{Generation Coverage} is the mean unknown-region mask ratio, a descriptor of synthesis difficulty rather than an accuracy metric.

Relative to GenDoP, the reported OmniCam results improve CLIP score by 14.1\%, consecutive-frame SSIM by 15.4\%, the aesthetic proxy by 18.0\%, and known-region fidelity by 9.2\%. The unknown-region mask ratio is 38.5\% vs.\ 35.8\%. These results describe outputs of the trajectory-plus-video pipeline; they do not isolate framing, path length, speed, or geometric consistency as the cause of improvement. Quantitative reconstruction accuracy and completeness are not evaluated in this table.

\noindent\textbf{Robotic Manipulation (Active Perception).}
OmniCam is evaluated as an active-perception module for $\pi_{0.5}$~\cite{intelligence2025pi}, generating camera trajectories to improve target observations. A trial succeeds when the designated object is grasped within the time limit. The original record describes 200 episodes across 50 tabletop scenes, with reported success rates of 72.5\% for OmniCam, 48.0\% for GenDoP, and 56.3\% for a fixed camera. The fixed-camera percentage is incompatible with a single pooled denominator of 200 binary trials; its aggregation remains unresolved.

\Cref{fig:real_robot} presents qualitative physical-robot grasping examples with initially occluded targets. The demonstration illustrates feasibility in the shown cases. A systematic sim-to-real claim would require repeated physical trials, matched fixed-camera controls, and documented camera motions and failure counts.

\section{Statistical Confidence Analysis}
\label{sec:confidence}

\Cref{tab:confidence} reports 95\% confidence intervals from 1,000 bootstrap resamples of OmniCaT evaluation scenes for OmniCam and retrained GenDoP. The listed intervals do not overlap. These estimates characterize variation across evaluation scenes, not variation across training seeds, and do not cover all other experiments. GenDoP is the comparison in this table, but is not the best baseline for every main-table metric.

\begin{table*}[tp]
\centering
\small
\caption{\textbf{Bootstrap 95\% confidence intervals on OmniCaT evaluation scenes} (1,000 resamples). Intervals describe evaluation-scene variability for the two listed methods.}
\label{tab:confidence}
\resizebox{\linewidth}{!}{
\begin{tabular}{l|cc|cc|c|cc}
\toprule
\textbf{Method} & \textbf{ATE $\downarrow$} & \textbf{FDE $\downarrow$} & \textbf{RPE-R $\downarrow$} & \textbf{RPE-T $\downarrow$} & \textbf{Coll.~(\%) $\downarrow$} & \textbf{Vis.~(\%) $\uparrow$} & \textbf{Angle~($^{\circ}$) $\downarrow$} \\
\midrule
GenDoP (OmniCaT) & 1.624{\scriptsize$\pm$0.087} & 2.483{\scriptsize$\pm$0.124} & 1.577{\scriptsize$\pm$0.068} & 0.038{\scriptsize$\pm$0.003} & 30.4{\scriptsize$\pm$1.8} & 33.6{\scriptsize$\pm$2.1} & 27.6{\scriptsize$\pm$1.4} \\
\cellcolor[HTML]{efefef}\textbf{OmniCam (Ours)} & \cellcolor[HTML]{efefef}\textbf{0.868}{\scriptsize$\pm$0.042} & \cellcolor[HTML]{efefef}\textbf{1.462}{\scriptsize$\pm$0.071} & \cellcolor[HTML]{efefef}\textbf{1.135}{\scriptsize$\pm$0.053} & \cellcolor[HTML]{efefef}\textbf{0.027}{\scriptsize$\pm$0.002} & \cellcolor[HTML]{efefef}\textbf{10.4}{\scriptsize$\pm$1.1} & \cellcolor[HTML]{efefef}\textbf{52.0}{\scriptsize$\pm$1.9} & \cellcolor[HTML]{efefef}\textbf{10.9}{\scriptsize$\pm$0.8} \\
\bottomrule
\end{tabular}
}
\end{table*}

\end{document}